\documentclass[wcp,cleveref]{jmlr}

\usepackage[T1]{fontenc}

\newif\ifanon
\anonfalse

\usepackage{booktabs}
\usepackage{enumitem}
\providecommand{\qedhere}{}

\newcommand{\Deff}{D_{\mathrm{eff}}}
\newcommand{\Neff}{N_{\mathrm{eff}}}
\newcommand{\betamix}{\beta}
\newcommand{\CKA}{\mathrm{CKA}}
\newcommand{\BI}{\mathrm{BI}}
\newcommand{\RR}{\mathbb{R}}
\newcommand{\EE}{\mathbb{E}}

\newcommand{\lead}[1]{\noindent\textbf{#1}\enspace}

\renewcommand{\rho}{\varrho}

\jmlryear{2026}
\jmlrworkshop{ACML 2026}
\editors{Andy Song, Bo Han and Sarah Erfani}

\title[Residual Stream's Depth]{The Residual Stream's Effective Depth}

\author{%
  \Name{Barak Gahtan} \Email{barakgahtan@cs.technion.ac.il}\\
  \addr Technion Israel Institute of Technology
  \AND
  \Name{Ido Galil} \Email{igalil@nvidia.com}\\
  \addr Nvidia
  \AND
  \Name{Alex M. Bronstein} \Email{bron@cs.technion.ac.il}\\
  \addr Technion Israel Institute of Technology \& ISTA Institute of Science and Technology Austria}

\begin{document}

\maketitle

\begin{center}
  \itshape Accepted for publication at the Asian Conference on Machine
  Learning (ACML 2026); to appear in Proceedings of Machine Learning
  Research (PMLR).
\end{center}

\begin{abstract}
We introduce \emph{effective depth} ($\Deff$), a scalar diagnostic that treats the layer-wise residual stream of a transformer as a discrete-time process, measures how representation similarity decays with layer distance, and aggregates that profile into one number.  Across sixteen decoder-only language models, $\Deff$ separates a structural consequence of residual accumulation from an empirical one: even maximally diverse orthogonal updates have the closed-form reference $F_L = 2L/(L+1)<2$, yet fifteen of sixteen default measurements lie below $F_L$ (Qwen3.5: 32--44\%, OLMo-2: 40--41\%, Pythia: 23--28\%).  Matched references show that the gap is not caused by the persistent initial state or update-size imbalance, but is largely a calibrated signature of correlated residual updates rather than evidence that depth is unused.  Symmetric position-0, token-normalisation, and top-PC controls show the regime is not reducible to BOS or top-PC artefacts: the lone above-reference default outlier joins the same regime, and all sixteen models are sub-reference after token-normalisation or top-1-PC removal.  Intermediate checkpoints show that the regime is established early in OLMo-2 and stable through 5T tokens, while Pythia-1.4B follows a distinct decreasing trajectory.  A controlled residual-carry intervention supports the mechanism, and $\Deff$ is best read as a \emph{global} accumulated-state diagnostic, not as a capability score or pruning method.
\end{abstract}

\begin{keywords}
  effective depth; residual stream; transformer interpretability; layer
  redundancy; centered kernel alignment; large language models
\end{keywords}

\section{Introduction}
\label{sec:intro}

The depth of a neural network is its most celebrated architectural parameter.
Theoretical analyses of deep learning routinely invoke depth as a proxy for
expressive power~\citep{telgarsky2016benefits,raghu2017expressive}, and
practitioners add layers as a first resort when models underfit.
Yet in the transformer era, a growing body of evidence suggests that
nominal depth $L$ is a poor surrogate for the number of statistically
distinguishable residual states a model actually traverses.
Studies of layer pruning~\citep{men2024shortgpt,gromov2024unreasonable},
representation similarity~\citep{kornblith2019similarity,raghu2021vision},
and skip-layer routing~\citep{elhoushi2024layerskip} collectively demonstrate
that large portions of a trained stack can be removed with negligible performance
degradation, while
\citet{veit2016residual} showed that residual networks behave as ensembles of
shallow networks whose effective gradient paths are far shorter than the
nominal depth.
Recent work confirms this pattern in modern LLMs: \citet{sun2025curse} show
that Pre-LayerNorm causes deeper layers to converge toward identity
transformations, and \citet{razzhigaev2024linear} find near-perfect linear
relationships (Procrustes similarity~$>0.99$) between adjacent layer
representations.
This raises a fundamental question:

\medskip
\begin{quote}
\itshape
Given a transformer with $L$ layers, how many statistically distinguishable
residual states does it traverse, and why?
\end{quote}
\medskip

Answering this question requires a definition that (i) is grounded in a
well-motivated diagnostic framework rather than ad hoc heuristics,
(ii) applies uniformly across decoder-only language models, and
(iii) connects to interpretable quantities such as layer redundancy.

\lead{Our approach.}
We recast the residual stream $(h_1, \dots, h_L) \in \RR^{L \times d}$
as a \emph{discrete-time process} indexed by layer, measure its
similarity autocorrelation $\hat\rho(k)$, and aggregate the full
layer-distance profile into a scalar.
We measure $\Deff$ on the \emph{accumulated} hidden states $h_\ell$ rather
than the per-layer updates $f_\ell = h_\ell - h_{\ell-1}$, because the
question we address concerns the distinct information states that downstream
computation receives (Section~\ref{sec:method}):
\begin{equation}
  \Deff \;=\; \frac{L}{1 + 2\sum_{k=1}^{L-1}\!\left(1-\tfrac{k}{L}\right)\hat\rho(k)}.
  \label{eq:deff_main}
\end{equation}
When layers are independent ($\hat\rho(k)=0$), $\Deff = L$;
when perfectly correlated ($\hat\rho(k)=1$), $\Deff = 1$.
This formulation directly connects the degree of inter-layer redundancy
to a single interpretable scalar.  The diagnostic has \textbf{four concrete uses:}
it prevents the mistaken reading that low normalized depth means a model
uses only a few layers; gives a global geometry summary for comparing
families, checkpoints, and architectural variants; surfaces
position-0, norm-spike, and top-PC anomalies when paired with controls;
and separates model-level residual-stream behaviour from local layer
decisions such as pruning.  In this sense, $\Deff$ is best read as a
calibrated diagnostic for residual-stream geometry: it tells us when a
stack follows the familiar residual-accumulation regime and when a new
architecture or checkpoint moves off that map.

\lead{What this paper does and does not claim.}
A reader who sees a normalized ratio of $\Deff/L < 0.05$ across 7B+ models
might infer that trained transformers ``waste'' most of their depth.
That inference is not what we report.  In an idealized
residual-accumulation regime with mutually orthogonal per-layer updates,
the closed-form value is $\Deff = 2L/(L+1) < 2$ at every $L$
(Section~\ref{sec:experiments},
Appendix~S2), so any model whose accumulated
states approximate this regime will have $\Deff/L = O(1/L)$ by
construction.  We use this value as an idealized
maximally-diverse-update reference, not as a null distribution or a
target.

The empirical content is the additional gap below this
reference, but this gap should be decomposed rather than overread.  We
therefore add matched references that retain measured update-size
profiles, the persistent initial state, and the measured update
similarity profile.  These controls show that the sub-$F_L$ gap is not
caused by $h_0$ carry or update-size imbalance; once measured update
similarities are retained, the scalar reference is more redundant than
the observed accumulated stream.  We therefore lead with absolute
$\Deff$ and gap-to-$F_L$, but interpret the small ratio as a calibrated
signature of correlated residual updates rather than as evidence of
learned underuse of layers.  CKA also makes $\Deff$ insensitive to pure
orthogonal rotations of the residual basis, so the diagnostic measures
accumulated-state similarity, not every possible form of computation.
We restrict empirical claims to decoder-only autoregressive language
models throughout; encoder-only, vision, and multimodal transformers are
left to future work.

\lead{Contributions.}
\begin{enumerate}
  \item \textbf{Effective depth ($\Deff$):} a scalar,
        architecture-agnostic diagnostic of redundancy in the
        accumulated residual stream, built from a CKA-based similarity
        autocorrelation and a weighted full-lag aggregation, both
        validated independently
        (Section~\ref{sec:method},
        Appendix~S13,
        Appendix~S7).
  \item \textbf{A calibrated sub-reference operating regime.}  Across
        sixteen decoder-only language models spanning four families,
        nearly every default measurement shows additional redundancy
        below the orthogonal-update reference $F_L = 2L/(L+1)$.
        Matched references show that this gap is not explained by the
        persistent initial state or measured update-size profile, but
        is largely accounted for by measured update correlations.  The
        same regime survives position-0, token-normalisation, and
        top-PC controls applied symmetrically to both outliers and
        non-outliers; the lone above-reference default outlier joins
        the regime under every control
        (Section~\ref{sec:experiments},
        Figure~\ref{fig:regime_gap_map},
        Figure~\ref{fig:control_variant_gap_map};
        Appendices~S3
        and~S4).
  \item \textbf{Training-dynamics arc.}  Three OLMo-2 sizes (1B, 7B,
        13B) establish a $40$--$41\%$ gap to $F_L$ by their first
        landed post-initialisation checkpoints and remain near it
        through 5T tokens, while
        Pythia-1.4B follows a distinct monotonically-decreasing
        trajectory ($44\%$ at 2B tokens to $26\%$ at 300B tokens).
        The sub-reference regime is therefore family-conditioned and
        early-established rather than solely a function of training
        progress (Section~\ref{sec:experiments},
        Figure~\ref{fig:training_dynamics}).
  \item \textbf{Mechanism and global-vs-local distinction.}  A
        controlled residual-carry intervention on a 12-layer nanoGPT
        shows that mild reductions in residual carry leave $\Deff(h)/L$
        unchanged, while the only reductions large enough to raise it
        substantially also break optimisation
        (Appendix~S17).  $\Deff$ is a
        \emph{global} accumulated-state diagnostic, not a capability
        score or pruning method; Appendix~S19
        reports pruning boundary checks that keep this scope explicit.
\end{enumerate}

\section{Related Work}
\label{sec:related}

\lead{Layer redundancy and pruning.}
\citet{men2024shortgpt} use Block Influence (BI), the cosine
similarity between input and output of each block, to score
prunability; \citet{gromov2024unreasonable} remove up to 50\% of
layers with small perplexity loss; \citet{elhoushi2024layerskip},
\citet{kim2024shortened}, \citet{raposo2024mixture}, and
\citet{fan2019reducing} leverage the same redundancy via early-exit,
structured removal, conditional routing, or stochastic dropout.
These works identify redundancy through ablation; $\Deff$ formalises
it as a single scalar that aggregates the full lag spectrum, of which
BI is the lag-1 term.

\lead{Depth utilization and residual structure.}
\citet{sun2025curse} identify the ``curse of depth'' in Pre-LayerNorm
architectures, where the Jacobian of deep blocks approaches identity.
\citet{razzhigaev2024linear} find Procrustes similarity $> 0.99$
between adjacent layers that drops substantially when the residual
component is removed; \citet{csordas2025depth} measure update norms
in Llama 3.1 and Qwen 3 and find second-half layers contribute far
less than first-half layers; \citet{jastrzebski2018residual}
distinguish representation learning (early) from iterative inference
(later); \citet{lad2024remarkable} argue that decoder-only LMs pass
through distinct inference stages; and \citet{veit2016residual} show
ResNets behave as
ensembles of shallow networks.  Our null baseline complements these
by showing the redundancy is architectural rather than a training
pathology, and $\Deff^{\mathrm{update}}$
(Appendix~S16) extends the per-layer view into a
global scalar.

\lead{Mechanistic perspective and massive activations.}
\citet{elhage2021mathematical} formalise the residual stream as a
shared communication channel;
\citet{geva2021transformer,geva2022transformer} show feed-forward
layers act as key-value memories, and \citet{belrose2023eliciting}
show intermediate-layer predictions improve monotonically.
\citet{yomdin2024jump} use learned linear transformations to expose
early-layer predictions, and \citet{skean2024representation} study
when intermediate-layer representations outperform final-layer
features.
\citet{queipodellano2026attention} show that massive position-0 / BOS
activations and within-layer compression valleys are two
manifestations of the same mechanism, operating within a single layer
and within a single sequence.  $\Deff$ measures redundancy of the
accumulated stream \emph{across depth} and \emph{across samples}; we
confirm in Section~\ref{sec:experiments}
(Figure~\ref{fig:control_variant_gap_map}) that position-0 controls
preserve the sub-reference regime rather than explain it away.

\lead{Representation similarity and intrinsic dimension.}
We use CKA~\citep{kornblith2019similarity} for its orthogonal
invariance; alternatives include
SVCCA~\citep{raghu2017expressive}, PWCCA~\citep{morcos2018insights},
and rotation-sensitive shape metrics
\citep{williams2021generalized}.
\citet{raghu2021vision,nguyen2021do} apply CKA to vision transformers
vs.\ CNNs and to language models.  Our aggregations connect to the
intrinsic-dimension literature~\citep{ansuini2019intrinsic}; the
Bartlett aggregation is standard in econometrics
\citep{newey1987simple} and MCMC diagnostics
\citep{geyer1992practical} and we use it as a scalar summary of the
layer autocorrelation profile, not as an estimator of any specific
mixing coefficient.

\section{Background}
\label{sec:background}

We collect the two ingredients used to define $\Deff$.

\lead{Centered Kernel Alignment.}
For representation matrices $H, H' \in \RR^{n \times d}$ with
centered columns, the linear CKA
similarity~\citep{kornblith2019similarity} is
\begin{equation}
  \CKA(H, H') \;=\; \frac{\|H'^\top H\|_F^2}{\|H^\top H\|_F \cdot \|H'^\top H'\|_F}.
  \label{eq:cka}
\end{equation}
$\CKA \in [0, 1]$ is invariant to isotropic scaling and orthogonal
transformation, and equals $1$ iff $H' = c H Q$ for scalar $c > 0$
and orthogonal $Q$.

\lead{Weighted full-lag aggregation.}
For a stationary scalar sequence with autocorrelation $\rho(k) =
\mathrm{Corr}(X_t, X_{t+k})$, lag weighting appears in
classical long-run-variance and sampling-efficiency
estimators~\citep{newey1987simple,geyer1992practical}:
\begin{equation}
  \Neff \;=\; \frac{N}{1 + 2\sum_{k=1}^{N-1}\!\left(1 - \frac{k}{N}\right)\rho(k)},
  \label{eq:ess}
\end{equation}
We borrow only the finite-lag weighting as a scalar summary
of the layer-similarity profile.  Trained models are not stationary
chains, so the strict ESS interpretation does not apply and $\Deff$ is
not a beta-mixing estimator.  Connections to mixing theory under
additional norm-process assumptions are deferred to
Appendix~S8.

\section{Method}
\label{sec:method}

\lead{Effective Depth ($\Deff$).}
Consider a transformer with $L$ layers; let $H_\ell \in \RR^{n
\times d}$ denote the matrix of hidden states at layer $\ell$ over
$n$ input sequences.  We exclude the embedding layer ($\ell = 0$) and
define the \emph{layer-wise process} $(H_1, \ldots, H_L)$.
Computing $\Deff$ requires a similarity metric between layer
representations and an aggregation that converts the resulting
profile into a scalar.

\lead{Two views of the residual stream.}
\label{par:two_views}
The residual connection $h_\ell = h_{\ell-1} + f_\ell(h_{\ell-1})$
admits two natural objects: the \emph{accumulated representation}
$h_\ell$ and the \emph{layer update} $f_\ell = h_\ell - h_{\ell-1}$.
We deliberately measure $\Deff$ on $h_\ell$ because pruning and
layer-skipping operate on the accumulated stream, and because the
question ``how many distinct information states does the network pass
through?'' concerns $h_\ell$.  Even maximally diverse updates produce
an accumulated-state $\Deff = 2L/(L+1) < 2$ in the idealised
construction with zero update cross-covariance
(Appendix~S2); this is a geometric property
of accumulation, not a statement about what each layer computes.  A
complementary diagnostic on $f_\ell$, $\Deff^{\mathrm{update}}$, is
reported in Appendix~S16.

\begin{definition}[Layer similarity autocorrelation]
\label{def:rho}
The \emph{layer similarity autocorrelation at lag $k$} is
\begin{equation}
  \hat\rho(k) \;=\; \frac{1}{L-k}\sum_{\ell=1}^{L-k} \CKA(H_\ell,\, H_{\ell+k}),
  \label{eq:rho_k}
\end{equation}
with $\CKA$ as in Eq.~\eqref{eq:cka} on column-centered inputs;
$\hat\rho(k) \in [0,1]$.
\end{definition}

CKA is invariant to orthogonal transformation and isotropic scaling;
we verify in Appendix~S13 that unbiased
HSIC-CKA~\citep{song2012feature} and cosine yield nearly identical
$\hat\rho(k)$.  $\hat\rho$ is the analogue of scalar autocorrelation
along the layer dimension; we do not claim it estimates any specific
mixing coefficient.

\begin{definition}[Effective depth]
\label{def:deff}
Given $L$ layer representations and $\hat\rho(k)$,
\begin{equation}
  \Deff \;=\; \frac{L}{1 + 2\sum_{k=1}^{L-1}\!\left(1 - \frac{k}{L}\right)\hat\rho(k)}.
  \label{eq:deff}
\end{equation}
The denominator lies in $[1, L]$, so $\Deff \in [1, L]$ without
clipping.
\end{definition}

The Bartlett taper $(1 - k/L)$ down-weights high-lag terms, so the
sum runs over all $L-1$ lags without bandwidth selection.
Alternative aggregations, e.g.\ participation
ratio~\citep{roy2020effective} and effective
rank~\citep{roy2007effective}, give independent summaries of the same
data and yield qualitatively identical conclusions
(Appendix~S7).  $\Deff$ is a
Bartlett-weighted full-lag scalar summary, not a strict sampling-efficiency estimate
(trained models are not stationary chains; see
Section~\ref{sec:experiments}).

\begin{proposition}[Orthogonal-update reference]
\label{prop:fl}
Let $h_\ell = \sum_{j=1}^\ell f_j$ with per-layer updates $f_j$
satisfying common covariance and zero cross-covariance across
layers.  Then the Bartlett-aggregated effective depth equals
\[
  F_L \;:=\; \Deff \;=\; \frac{2L}{L+1} \;<\; 2
  \qquad \text{for all } L \geq 1.
\]
\end{proposition}

We use $F_L$ as an idealized maximally-diverse-update reference
throughout Section~\ref{sec:experiments}.  It is not a null
distribution or a target value; the measured gap to $F_L$ quantifies
additional accumulated-state redundancy beyond this structural
residual-accumulation reference.  We further calibrate this gap using
matched scalar references that retain measured update-size profiles,
$h_0$ carry, and measured update-similarity profiles
(Appendix~S3).  Proof in
Appendix~S2.

\begin{proposition}[Convergence of CKA-based estimator]
\label{prop:consistency}
Fix $\ell, m \in \{1, \ldots, L\}$ and feature dimension $d$.  If
inputs $x_1, \ldots, x_n$ are i.i.d., the activations satisfy
$\EE[\|h_j(x_i)\|^4] < \infty$, and the population covariance
matrices $\Sigma_{\ell\ell}, \Sigma_{mm}$ are nonzero, then
$\hat\rho_{\mathrm{CKA}}(k)$ is a consistent estimator of its
population counterpart as $n \to \infty$
(proof in Appendix~S5).
\end{proposition}

\section{Experiments}
\label{sec:experiments}

\lead{Models and setup.}
We evaluate sixteen publicly available decoder-only language models
across four architecture families:
\textbf{Pythia} \citep{biderman2023pythia}
(70M/160M/410M/1B/1.4B; $L \in \{6,12,16,24,24\}$, base),
\textbf{OLMo-2} \citep{groeneveld2024olmo}
(1B/7B/13B; $L \in \{16,32,40\}$, base),
\textbf{Qwen3.5} (0.8B/2B/4B/9B/27B; $L \in \{24,24,32,32,64\}$,
instruct), and individual base models Llama-3.1-8B
\citep{dubey2024llama}, Mistral-7B-v0.3 \citep{jiang2023mistral},
Gemma-3-12b-pt \citep{team2024gemma}
($L = 32, 32, 48$).  All are autoregressive LMs; encoder-only and
vision transformers are out of scope.  We extract mean-pooled hidden
states from $N{=}10{,}000$ FineWeb-Edu \citep{penedo2024fineweb}
passages per model ($N>d$ for every model; finite-sample behaviour is
calibrated in Appendices~S2 and~S13);
CKA is computed in float64 with Frobenius normalisation and $\Deff$
uses the full autocorrelation ($K = L{-}1$, no truncation).

\lead{Reference value.}
We compare each model against the \emph{orthogonal-update reference}
$F_L = 2L/(L+1) < 2$, the closed-form $\Deff$ when per-layer updates
have zero pairwise CKA and are accumulated into the residual stream
(Appendix~S2).  This is an idealized
maximally-diverse-update reference, not a probabilistic null or a
target.  The signed gap
$\mathrm{gap} = (F_L - \Deff)/F_L$ is positive when a model's
accumulated states are more similar than this residual-accumulation
reference would produce, and negative when they are less so.
The evidence below follows the intended use of $\Deff$: establish the
regime, calibrate the gap with matched references, test whether it
survives pooling and norm controls, trace when it appears during
training, probe the residual-carry mechanism, and separate global
redundancy from local pruning; we also include a conservative external
check against benchmark scaling.


\begin{figure}[t]
  \centering
  \includegraphics[width=0.98\linewidth]{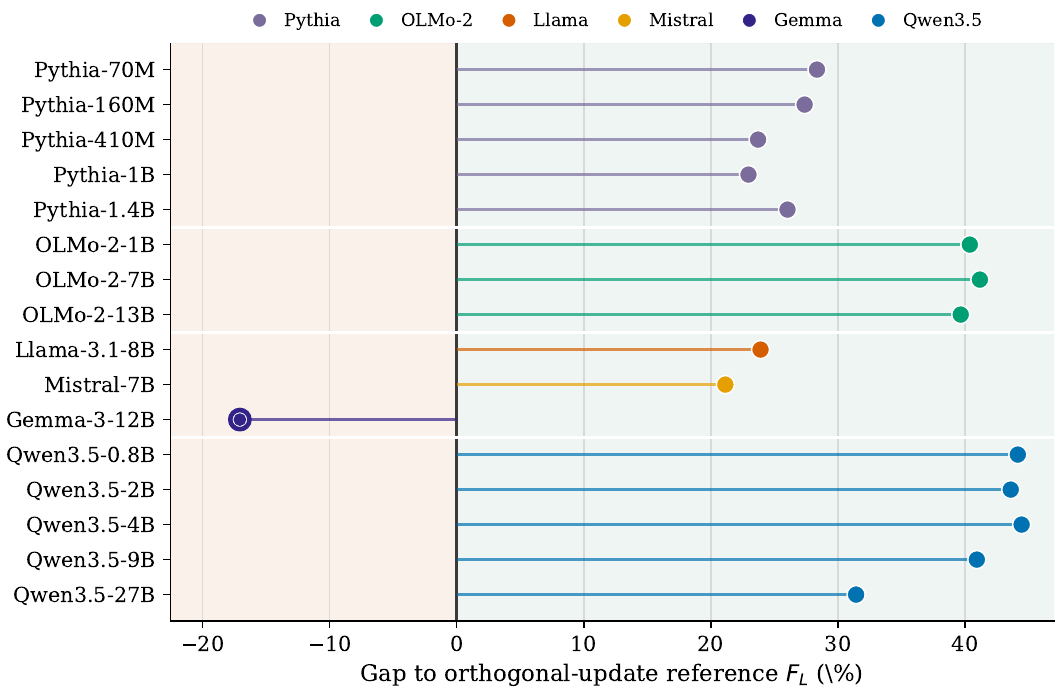}
  \caption{Gap to the orthogonal-update reference
           $F_L = 2L/(L+1)$ for the sixteen-model sweep
           ($N{=}10{,}000$, FineWeb-Edu, $K{=}L{-}1$).  Positive
           values indicate $\Deff < F_L$, i.e.\ accumulated states show
           additional redundancy beyond the orthogonal-update reference; negative
           values indicate $\Deff > F_L$.  Fifteen of sixteen default
           measurements are sub-reference, with gaps clustering by
           model family.  Gemma-3-12b is the lone above-reference
           default outlier; position-0 and norm controls in
           Figure~\ref{fig:control_variant_gap_map} move it into the
           sub-reference regime.  Full numeric values are reported in
           Table~\ref{tab:deff_models}.}
  \label{fig:regime_gap_map}
\end{figure}

\lead{A sub-reference operating regime.}
Fifteen of the sixteen models in Figure~\ref{fig:regime_gap_map} show
positive gap below $F_L$ of $20$--$44\%$, with gaps clustering by family (Qwen3.5
$32$--$44\%$, OLMo-2 $40$--$41\%$, Pythia $23$--$28\%$,
Llama-3.1-8B and Mistral-7B $21$--$24\%$).  Within each family the gap
is roughly constant in scale; Gemma-3-12b at $-17.6\%$ is the lone
above-reference default outlier.  This pattern suggests that
family-level design and training choices, rather than parameter count
alone, shape the residual-stream similarity profile.  Absolute $\Deff$ for 7B+ models
stays in $[1.15, 2.30]$ despite nominal depths $32$--$64$; the small
$\Deff/L$ is a diagnostic signature, not direct evidence of underuse,
since residual accumulation alone produces $\Deff/L = O(1/L)$ at every
depth (controls below).  Passage-bootstrap $95\%$ CIs collapse to the
reported point estimate at three-decimal precision
(Appendix~S18).

\lead{Matched references calibrate the gap.}
The universal $F_L$ reference deliberately removes all model-specific
update geometry, so we next ask which model-specific quantities explain
the gap.  Appendix~S3 reports three scalar
matched references.  Matching the measured update-size profile, with
or without retaining the persistent initial state $h_0$, does not
remove the sub-$F_L$ gap: all sixteen models remain below these
references.  The sign pattern is uniform: gaps to the update-size and
$h_0$-matched references are positive for every row ($+17$--$77\%$),
whereas gaps to the update-correlation-matched reference are negative
for every row ($-3$--$-89\%$).  Thus the gap to $F_L$ is not mainly an
$h_0$ or update-norm artifact; it is largely explained by correlated
updates.  The remaining difference has the opposite sign: real
accumulated states are more diverse than a scalar model that preserves
only update sizes and average update similarities.

\lead{Synthetic constructions verify the diagnostic.}
On controlled constructions evaluated with the same linear-CKA and
weighted aggregation pipeline (Appendix~S2), $\Deff/L =
0.258$ for independent states and $0.996$ for orthogonal states, so
the diagnostic rises when states are genuinely decorrelated.  Orthogonal
per-layer updates accumulated into states yield the closed-form
$\Deff = F_L = 2L/(L+1)<2$, so even maximally distinct updates produce
$\Deff/L=O(1/L)$ in the accumulated stream.  Low $\Deff(h)$ is
therefore not by itself evidence that specific layers are functionally
redundant, and CKA-invisible rotations remain outside the diagnostic's
scope.


\lead{Position-0 and norm controls: a consistent sub-reference regime.}
Recent work shows that decoder-only LMs develop massive position-0 /
BOS activations and within-layer compression
valleys~\citep{queipodellano2026attention}, which can in principle
distort any pooled diagnostic.  We re-evaluate $\Deff$ on a six-model
representative set under four position/norm \emph{control variants},
applying the same controls to default outliers and non-outliers:
pooling from position 1 only (``first non-BOS''),
$\ell_2$-normalising each token before pooling (``token-normed''),
and projecting out the top-1 or top-3 cross-sample principal
components per layer before CKA.  Figure~\ref{fig:control_variant_gap_map}
reports the resulting gap to $F_L$ for six representative models; the
exact values and full sixteen-model control table are given in
Appendix~S4.

The pattern is consistent but not invariant.  Token-normalisation and
top-PC removal leave OLMo-2 and Qwen3.5 sub-reference, while
single-position pooling can move Qwen3.5 substantially; this is why we
treat the controls as sensitivity analyses rather than as an alternate
canonical measurement.  Llama-3.1-8B is sub-reference under default
($+24\%$) but its tokenizer auto-prepends BOS and its position-0 norm
peaks at $310\times$ at layer~2; once BOS is excluded its gap rises to
$+41\%$, the same regime as OLMo-2 and Qwen3.5.  The single
above-reference model in the default measurement, Gemma-3-12b
($11\times$ position-0 norm at layer~7), shifts under \emph{every}
control variant from $-17\%$ to gaps in $[+27\%, +42\%]$; token
normalisation alone is sufficient.  Position-0 norm spikes therefore
modulate the measured gap, sometimes strongly, but they do not explain
the regime away: every model in the panel is sub-reference after
first-non-BOS, token-normalisation, or top-PC controls.  We retain
default pooling as the canonical measurement and use the controls to
separate the family-level regime from massive-activation geometry.
Across the full sixteen-model sweep, all models are sub-reference after
token-normalisation and after top-1-PC removal; first-non-BOS pooling
is also sub-reference except for the shallow Pythia-70M boundary case
(Appendix~S4).
Additional pretrained Gemma-3 sizes show that this behaviour is family-level:
Gemma-3-4B is near-reference and Gemma-3-27B above-reference by
default, but controls move both sub-reference; Gemma-scale
interventions show that post-normalisation removal collapses the
stream while QK-normalisation removal amplifies the default-pooling
anomaly.  Initialisation-only and short-pretraining probes do not
reproduce the anomaly, so we treat these norm effects as modulators of
trained checkpoints rather than as a complete architectural explanation
(Appendix~S4).

\begin{figure}[t]
  \centering
  \includegraphics[width=\linewidth]{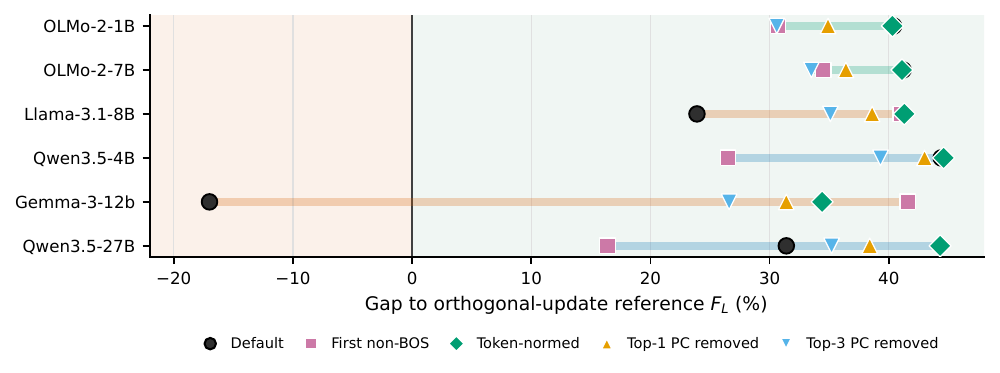}
  \caption{Position-0 and norm control variants for a six-model
           representative set.  Points show gap to $F_L$ in percent;
           the vertical line marks the reference.  ``First non-BOS''
           uses position 1 only, ``token-normed'' normalises each token
           before pooling, and top-PC controls remove the leading
           cross-sample components before CKA.  The controls can move
           the measured gap substantially, especially for Gemma-3-12B
           and BOS-sensitive Llama/Qwen cases, but all controlled
           variants in this panel are sub-reference.  Exact values and
           the full sixteen-model table are in
           Appendix~S4.}
  \label{fig:control_variant_gap_map}
\end{figure}


\lead{Training dynamics.}
To test whether the sub-reference regime is a late-training endpoint
or an early property of the residual stream, we evaluate $\Deff$
across intermediate OLMo-2 checkpoints and a Pythia-1.4B trajectory
(Figure~\ref{fig:training_dynamics}).  Across three OLMo-2 sizes the
gap to $F_L$ is established early and stays near it: a $\sim 40\%$
gap by the first landed post-initialisation OLMo-2-1B checkpoint
(21B tokens), $40$--$41\%$ over 412B--3.5T tokens
for OLMo-2-7B, and $40$--$43\%$ from initialisation through 5T tokens
for OLMo-2-13B.  Pythia-1.4B follows a different trajectory,
decreasing monotonically from $\sim 44\%$ at 2B tokens to $26\%$ at
300B tokens.  The regime is therefore not solely a function of
training progress but a stable family-level operating point in
OLMo-2, with older architectures following a distinct path; we do not
attribute the contrast to any single architectural choice without the
controlled probe in Appendix~S17.

\begin{figure}[t]
  \centering
  \includegraphics[width=0.95\linewidth]{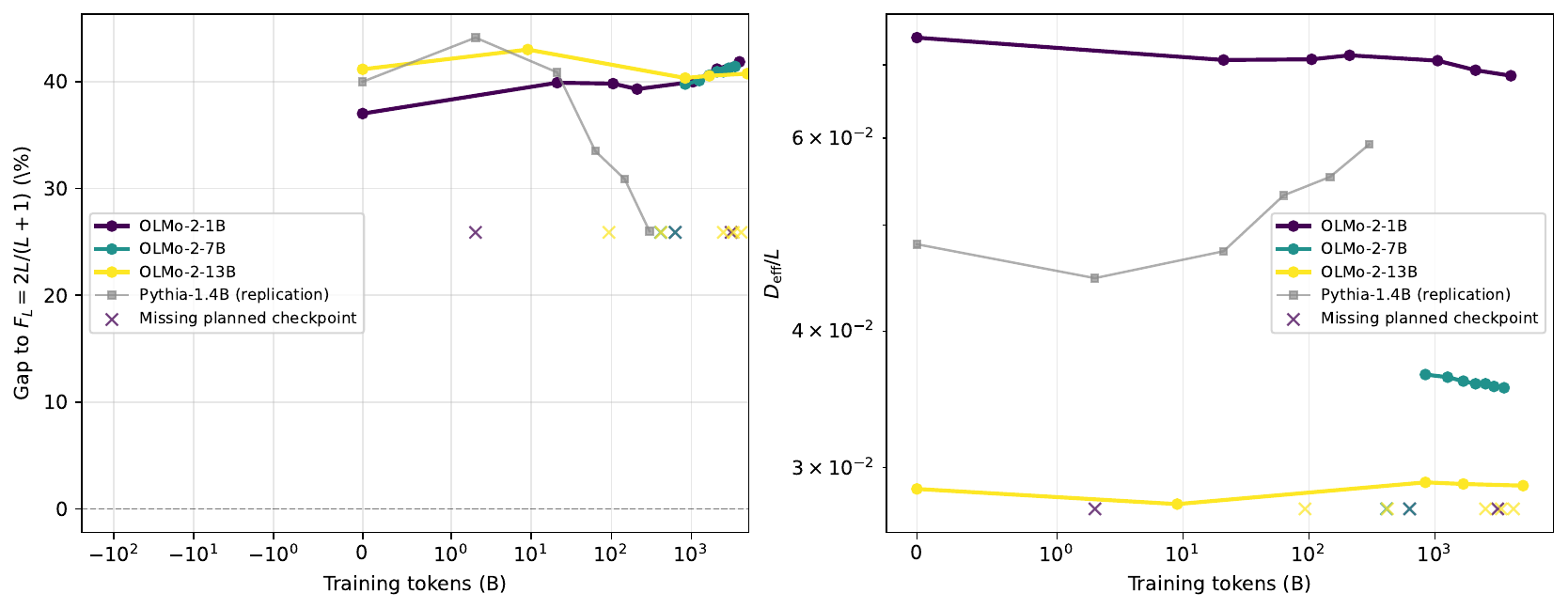}
  \caption{Effective-depth training dynamics.  Left: gap to $F_L =
           2L/(L+1)$ as a percentage, vs.\ training tokens.  Right:
           $\Deff/L$ vs.\ training tokens.  Three OLMo-2 sizes
           (1B, 7B, 13B) converge to a stable $40$--$41\%$ gap by
           the earliest landed post-initialisation checkpoints and
           persist through 5T tokens; Pythia-1.4B shows a distinct
           monotonic decrease.  The figure plots all landed
           checkpoints (25 of 35 planned); missing OLMo-2 points
           correspond to unavailable or failed revisions and are not
           interpolated.}
  \label{fig:training_dynamics}
\end{figure}


\lead{Null baseline: a strong residual-geometric component.}
We compute $\Deff$ for both trained and \textbf{random-weight}
versions of the same architectures
(Table~\ref{tab:null_multi}, Appendix~S10).
Trained and random-weight $\Deff/L$ remain in the same low range
(maximum absolute difference $0.018$), indicating a strong
residual-geometric component.  A five-seed seed-stability check on
seven architectures shows that random-weight $\Deff/L$ is seed-stable
with standard deviation at most $2 \times 10^{-4}$,
defining a near-deterministic architectural reference; trained values
lie outside this random-seed range in all seven cases, with the
direction architecture-dependent (trained values are below random for
Qwen3.5 and OLMo-2-1B, but above random for larger OLMo-2 and
Llama-3.1-8B; full numbers in
Appendix~S10).  The architecture and initialisation
largely determine the order of magnitude of accumulated-state
effective depth; training perturbs $\Deff/L$ modestly without changing this scale, so
$\Deff$ is primarily sensitive to residual-stream geometry rather than
to the full content of learned computations.

\begin{table}[t]
  \centering
  \caption{Multi-family null baseline ($N{=}10{,}000$, FineWeb-Edu).
           Trained and random-weight $\Deff/L$ remain in the same low
           range across model families for completed pairs (maximum
           absolute difference 0.018).  For OLMo-2-1B/7B/13B,
           Qwen3.5-4B/9B/27B, and Llama-3.1-8B, the random column
           reports the mean of five independent seeds; for the other
           rows it is a single architecture-default initialisation.}
  \label{tab:null_multi}
  \begin{tabular}{lcccc}
    \toprule
    Model & $L$ & Trained & Random & $|\Delta|$ \\
    \midrule
    OLMo-2-1B      & 16 & 0.070 & 0.074 & 0.004 \\
    Qwen3.5-0.8B   & 24 & 0.045 & 0.060 & 0.016 \\
    Qwen3.5-2B     & 24 & 0.045 & 0.063 & 0.018 \\
    Qwen3.5-4B     & 32 & 0.034 & 0.047 & 0.013 \\
    Llama-3.1-8B   & 32 & 0.046 & 0.035 & 0.011 \\
    OLMo-2-7B      & 32 & 0.037 & 0.036 & 0.001 \\
    Qwen3.5-9B     & 32 & 0.036 & 0.045 & 0.009 \\
    OLMo-2-13B     & 40 & 0.029 & 0.029 & 0.001 \\
    Qwen3.5-27B    & 64 & 0.021 & 0.022 & 0.001 \\
    \bottomrule
  \end{tabular}
\end{table}


\lead{Capability-relevant scaling within Qwen3.5.}
To test whether $\Deff$ is more than an internal measurement, we compare
it with external benchmarks.  Within Qwen3.5, lower $\Deff/L$ co-moves
with higher MMLU ($r=-0.896$, $p=0.040$) and ARC-Challenge
($r=-0.913$, $p=0.030$); HellaSwag is borderline.  Thus $\Deff$ tracks
a capability-relevant scaling trajectory in this family, although the
same axis also changes scale and depth.  We therefore do not treat
$\Deff$ as a capability score, and cross-family panels remain mixed
(Appendix~S20).


\lead{Residual-carry mechanism.}
The natural mechanism is the residual skip itself: when each layer
adds an update of small relative norm, $h_\ell \approx h_{\ell-1}$
regardless of what the layer computes.  Two complementary supports
agree.  First, the per-layer-update diagnostic
$\Deff^{\mathrm{update}}/L$ stays in the same low band as $\Deff/L$
across all evaluated models (Appendix~S16,
Table~S13), so the updates themselves are
correlated.  Second, a controlled $\gamma$-sweep on a 12-layer
nanoGPT (Appendix~S17) shows that mild
reductions in residual carry leave $\Deff(h)/L$ unchanged, while the
only reductions that substantially raise it also break optimisation.
These results support residual carry as a causal contributor and
suggest that escaping the low-$\Deff$ regime is tied to stable
optimisation in this controlled setting.


\lead{$\Deff$ vs.\ BI: global summary vs.\ local adjacency score.}
$\Deff$ aggregates full-lag accumulated-state similarity via the
Bartlett taper; the closest single-lag local diagnostic is the
Block-Influence (BI) score~\citep{men2024shortgpt}, exactly $1$ minus
the lag-1 cosine similarity between adjacent layers.  The two views
agree numerically on many deep, lag-1-dominated models, which is
expected; Appendix~S15 shows that the full-lag
aggregation matters in shallow boundary cases such as Pythia-70M, where
higher-lag decay would be missed by lag-1 alone.  Thus $\Deff$ is not a
replacement for BI: BI scores individual layers, while $\Deff$ is a
calibrated full-profile model summary with the $F_L$ reference.
We therefore do not present $\Deff$ as a pruning method.  For
completeness, Appendix~S19 reports boundary
checks showing that local $\Delta\Deff$ is a poor layer-pruning ranker
relative to BI, and that an exploratory model-level pruning-tolerance
correlation is not robust enough to support a main claim.

\lead{Robustness and within-family.}
Alternative similarity metrics, aggregations, lag-truncation choices,
and tapers all preserve the qualitative ordering
(Appendices~S13,~S7,~S6,~S9).
Appendix~S11 gives a path-length companion: geodesic
efficiency correlates only moderately with $\Deff/L$ ($r=0.65$), so it
is related to but not redundant with the CKA-based summary.  Within
Qwen3.5, absolute $\Deff$ changes only fractionally with scale
($1.08$--$1.15$ over the 24/32-layer models, $1.34$ at 64 layers), while
Llama-3.1-8B and Mistral-7B sit closer to $F_L$ than Qwen3.5 or OLMo-2
at comparable depth; cross-family comparisons should therefore use
gap-to-reference rather than $\Deff/L$ alone.

\section{Discussion}
\label{sec:discussion}

\lead{What $\Deff$ measures, and what the regime tells us.}
$\Deff$ is a \emph{global} summary of redundancy in the accumulated
residual stream: it rises when states are genuinely decorrelated
($0.996$ on orthogonal, $0.258$ on independent) and falls toward
$F_L$ when accumulated states approximate maximally diverse residual
accumulation.  Because $F_L=2L/(L+1)$ is itself structural, the empirical
claim is the additional measured gap and its explanation.  Matched
references show that the gap survives update-size and $h_0$ matching but
is accounted for by measured update similarity, exposing a
correlated-update regime rather than a hidden count of unused layers.
Position-0, token-normalisation, top-PC controls, and training dynamics
then show that this regime is neither a pooling artifact nor solely a
function of training progress; in the evaluated models it behaves like a
family-conditioned operating point.  The practical value is therefore
calibration: $\Deff$ tells us when low normalized depth is expected from
residual accumulation, when a family sits unusually far below or above
that reference, and when a control reveals a position-0 or norm-driven
distortion rather than a change in layer computation.

\lead{Mechanism: residual carry, interpreted once.}
A natural mechanism is the residual skip itself: when each layer adds a
small relative-norm update, $h_\ell \approx h_{\ell-1}$ regardless of the
layer computation.  The update-level diagnostic remains in the same low
band as $\Deff/L$ (Appendix~S16), and the 12-layer
nanoGPT $\gamma$-sweep shows that only large residual-carry reductions
escape the low-$\Deff$ regime, while also breaking optimisation
(Appendix~S17).  Residual carry is not the only
mechanism: pre-norm, norm-growth, and Procrustes effects
\citep{sun2025curse,razzhigaev2024linear} also compound with residual
accumulation, and the OLMo-2-vs-Pythia contrast still needs controlled
architectural ablation.  This is why we frame the result as a geometry
of the accumulated stream rather than as a claim that deeper layers do
not compute: updates can be meaningful while the states received by
downstream layers remain highly correlated.

\lead{Global vs.\ local: $\Deff$, BI, and prunable fraction.}
$\Deff/L \approx 0.03$--$0.05$ for 7B+ models, yet pruning studies
report 25--50\% removable
\citep{men2024shortgpt,gromov2024unreasonable}.  The apparent gap is
real: $\Deff$ counts representational diversity, while pruning measures
task-specific functional tolerance.  We therefore treat pruning only as
a boundary check (Appendix~S19): $\Deff$
summarises the model, while BI scores local layer adjacency.  This
division is useful in practice.  A new checkpoint, architecture, or
compressed model can be compared globally by its gap-to-reference and
control variants; local layer removal should still use local influence
or ablation scores.  Expecting one scalar to answer both questions is
precisely the interpretation error that the calibrated reference is
meant to prevent.

\lead{Limitations and future work.}
$\Deff$ is a corpus-based measurement, not an architectural prediction;
CKA is blind to pure residual-stream rotations, so Procrustes or shape
metrics~\citep{williams2021generalized} are natural complements; and
trained models are not stationary chains, so $\Deff$ is not a literal
ESS or beta-mixing estimator (Appendix~S8).
Empirical claims are restricted to decoder-only LMs on FineWeb-Edu.
The next step is a controlled grid varying residual update scale,
initialisation, normalisation, and architectures that weaken the unit
identity path.  Such experiments would turn the family-conditioned
patterns observed here into causal architectural statements, and would
clarify whether designs such as mixture-of-depths, residual scaling, or
parallel depth can deliberately move models away from the standard
residual-accumulation regime.

\section{Conclusion}
\label{sec:conclusion}

Effective depth ($\Deff$) is a global diagnostic of accumulated
residual-state redundancy with an explicit structural baseline:
orthogonal updates already give $F_L = 2L/(L+1)<2$.  Across sixteen
decoder-only LMs, the calibrated gap is mostly sub-reference, is tied
by matched references to correlated updates, and survives position-0,
token-normalisation, and top-PC controls.  Together with the training
dynamics, random-weight baselines, and residual-carry probe, this makes
$\Deff$ a model-level diagnostic of residual-stream geometry rather
than a claim that only a few layers matter.  Its main use is to give
researchers a calibrated map: where a model family sits relative to the
residual-accumulation reference, which anomalies are caused by
position-0 or norm geometry, and when global state redundancy should be
separated from local pruning decisions.  In that map, $\Deff$ is the
global coordinate and BI-like scores are the local companion.

\acks{We thank the anonymous ACML 2026 reviewers and area chairs for
constructive feedback that improved the paper. Computational resources were
provided by the Institute of Science and Technology Austria (ISTA) and the
Technion Israel Institute of Technology.}

\setlength{\bibsep}{1pt plus 0.3ex}
\bibliography{references}

\appendix

\section{Full Effective-Depth Measurements}
\label{app:full_deff_table}

Table~\ref{tab:deff_models} reports the numeric values underlying
Figure~\ref{fig:regime_gap_map}.

\begin{table}[t]
  \centering
  \caption{Effective depth ($N{=}10{,}000$, FineWeb-Edu, $K{=}L{-}1$).
           Absolute $\Deff$ remains $O(1)$ across scales.  All models
           except Gemma-3-12b sit \emph{below} the orthogonal-update
           reference $F_L = 2L/(L+1)$ (positive gap), with the
           gap clustering by family; Gemma-3-12b is the lone
           above-reference default outlier.}
  \label{tab:deff_models}
  \resizebox{\linewidth}{!}{%
  \begin{tabular}{lcccccc}
    \toprule
    Model & $L$ & $d$ & $\Deff$ & $\Deff / L$ & gap (\%) & $\hat\rho(1)$ \\
    \midrule
    Pythia-70M     & 6  & 512  & 1.23 & 0.205 & $+28.2$ & 0.862 \\
    Pythia-160M    & 12 & 768  & 1.34 & 0.112 & $+27.2$ & 0.856 \\
    Pythia-410M    & 24 & 1024 & 1.46 & 0.061 & $+23.8$ & 0.940 \\
    Pythia-1B      & 16 & 2048 & 1.46 & 0.091 & $+22.7$ & 0.914 \\
    Pythia-1.4B    & 24 & 2048 & 1.42 & 0.059 & $+26.2$ & 0.954 \\
    \midrule
    OLMo-2-1B     & 16 & 2048 & 1.12 & 0.070 & $+40.5$ & 0.972 \\
    OLMo-2-7B     & 32 & 4096 & 1.15 & 0.036 & $+40.6$ & 0.988 \\
    OLMo-2-13B    & 40 & 5120 & 1.18 & 0.029 & $+39.7$ & 0.988 \\
    \midrule
    Llama-3.1-8B   & 32 & 4096 & 1.47 & 0.046 & $+24.1$ & 0.964 \\
    Mistral-7B     & 32 & 4096 & 1.54 & 0.048 & $+20.8$ & 0.973 \\
    Gemma-3-12b    & 48 & 3840 & 2.30 & 0.048 & $-17.6$ & 0.910 \\
    \midrule
    Qwen3.5-0.8B   & 24 & 1024 & 1.08 & 0.045 & $+43.7$ & 0.986 \\
    Qwen3.5-2B     & 24 & 2048 & 1.08 & 0.045 & $+43.7$ & 0.984 \\
    Qwen3.5-4B     & 32 & 2560 & 1.09 & 0.034 & $+43.9$ & 0.985 \\
    Qwen3.5-9B     & 32 & 4096 & 1.15 & 0.036 & $+40.6$ & 0.974 \\
    Qwen3.5-27B    & 64 & 5120 & 1.34 & 0.021 & $+31.8$ & 0.973 \\
    \bottomrule
  \end{tabular}%
  }
\end{table}

\setcounter{section}{0}
\setcounter{table}{0}
\setcounter{figure}{0}
\setcounter{equation}{0}
\renewcommand{\thesection}{S\arabic{section}}
\renewcommand{\thetable}{S\arabic{table}}
\renewcommand{\thefigure}{S\arabic{figure}}
\renewcommand{\theequation}{S\arabic{equation}}
\renewcommand{\theHsection}{S\arabic{section}}
\renewcommand{\theHtable}{S\arabic{table}}
\renewcommand{\theHfigure}{S\arabic{figure}}
\renewcommand{\theHequation}{S\arabic{equation}}
\section*{Appendix Roadmap}

The appendix is organised as an evidence map for the main-text claims.
The full per-model numeric table underlying the main regime figure is
kept in the main paper (Appendix~A); this
supplement collects the remaining material.
Appendices~\ref{app:deff_properties}--\ref{app:synthetic_controls}
contain the mathematical and synthetic foundations
of $\Deff$: basic properties, the orthogonal-update reference $F_L$, and
controlled constructions that separate accumulated states from layer
updates.
Appendix~\ref{app:matched_references} gives the matched-reference
calibration that separates the universal $F_L$ gap from update-size,
$h_0$, and measured update-correlation effects.
Appendix~\ref{app:control_variants_all16} expands the position-0,
token-normalisation, and top-PC controls from the main representative
panel to all sixteen models, adds the Gemma-3 size extension, and
reports Gemma-style architectural probes at initialisation and under a
short pretraining budget.
Appendix~\ref{app:cka_convergence} gives the CKA convergence proof.
Appendices~\ref{app:k_sensitivity}--\ref{app:tapers} then collect
robustness checks for the aggregation itself, including lag truncation,
alternative aggregations, the norm-process mixing analogy, and
alternative tapers.  Appendix~\ref{app:exp_details} documents
extraction and CKA details and the trained-vs-random null baseline
figure, while Appendices~\ref{app:geodesic}--\ref{app:instruct_base}
report auxiliary diagnostics and representation-similarity robustness
checks.
Appendices~\ref{app:shortgpt_shallow}--\ref{app:residual_carry} support
the interpretation of $\Deff$ as a global residual-stream diagnostic:
they compare against ShortGPT-style shallow behaviour, report the
per-update diagnostic, and give the controlled residual-carry
intervention.  The final appendices provide statistical stability and
boundary checks: bootstrap confidence intervals
(Appendix~\ref{app:bootstrap_ci}), pruning boundary checks
(Appendix~\ref{app:pruning_tolerance}), and the external
capability-scaling check in Appendix~\ref{app:capability_correlation}.

\section{Properties of $\Deff$}
\label{app:deff_properties}

We state and prove the basic properties of the effective depth diagnostic.
Throughout, $\hat\rho(k)$ denotes the CKA-based layer similarity
autocorrelation (Definition~1 of the main paper), which satisfies
$\hat\rho(k) \in [0,1]$ since $\CKA \in [0,1]$.

\begin{theorem}[Properties of $\Deff$]
\label{thm:deff_properties}
Let $\hat\rho(k) \in [0,1]$ for all $k = 1, \ldots, L-1$.  Then:
\begin{enumerate}[label=(\roman*)]
  \item \textbf{Bounds:} $1 \leq \Deff \leq L$.
  \item \textbf{Independence limit:} If $\hat\rho(k) = 0$ for all $k \geq 1$,
        then $\Deff = L$.
  \item \textbf{Maximum redundancy:} If $\hat\rho(k) = 1$ for all $k$, then
        $\Deff = 1$ for every $L \geq 1$.
  \item \textbf{Algebraic structure:} $\Deff = L / \hat S$ where
        $\hat S = 1 + 2\sum_{k=1}^{L-1}(1-k/L)\hat\rho(k)$.  This has the same
        algebraic form as the Bartlett spectral density estimator at frequency
        zero, but we do not claim that $\hat\rho$ has the Toeplitz
        positive-semidefinite structure of a true autocorrelation sequence.
\end{enumerate}
\end{theorem}

\begin{proof}
\textbf{(i)} The denominator is
$\hat S = 1 + 2\sum_{k=1}^{L-1}(1-k/L)\hat\rho(k)$.
Since $\hat\rho(k) \geq 0$ and $(1-k/L) > 0$ for $k < L$, every term in the
sum is non-negative, so $\hat S \geq 1$ and therefore $\Deff \leq L$.
For the upper bound on $\hat S$: since $\hat\rho(k) \leq 1$,
\[
  \hat S \;\leq\; 1 + 2\sum_{k=1}^{L-1}\!\left(1-\frac{k}{L}\right)
  \;=\; 1 + 2\!\left((L-1) - \frac{(L-1)L/2}{L}\right)
  \;=\; L.
\]
Hence $\Deff = L / \hat S \geq 1$.

\textbf{(ii)} If $\hat\rho(k) = 0$ for $k \geq 1$, then $\hat S = 1$ and
$\Deff = L$.

\textbf{(iii)} If $\hat\rho(k) = 1$ for all $k$, then
\[
  \hat S = 1 + 2\sum_{k=1}^{L-1}\!\left(1-\frac{k}{L}\right)
  = 1 + 2(L-1) - 2 \cdot \frac{(L-1)L/2}{L}
  = 1 + 2(L-1) - (L-1) = L.
\]
So $\Deff = L/L = 1$ exactly.

\textbf{(iv)} This is immediate from the definition.
\end{proof}

\section{Controlled Synthetic Constructions and Empirical Validation}
\label{app:synthetic_controls}

To clarify what $\Deff$ does and does not measure, we evaluate controlled
synthetic constructions using the same linear-CKA and Bartlett aggregation
pipeline as the main experiments.  Unless otherwise noted, these controls use
$N = 10{,}000$, $L = 32$, and $d = 1024$.  The goal is qualitative validation
of the diagnostic: can it rise on genuinely decorrelated states, and can
accumulated states remain highly redundant even when the underlying updates
are diverse?  \Cref{tab:synthetic_controls} summarizes the three
constructions, and Figure~\ref{fig:synthetic_summary} contrasts them with
real models; the per-construction analyses below expand on each.

\begin{table}[t]
  \centering
  \caption{Controlled synthetic constructions.  Independent and orthogonal
           states show that the metric rises when states are genuinely
           decorrelated; orthogonal-update sweeps show that $\Deff(h)$ can
           remain low even when the updates themselves are diverse.}
  \label{tab:synthetic_controls}
  \begin{tabular}{lccc}
    \toprule
    Construction & $\Deff(h)/L$ & $\Deff(f)/L$ & $\hat\rho_h(1)$ \\
    \midrule
    Independent states & 0.258 & -- & 0.093 \\
    Orthogonal states & 0.996 & -- & 0.0001 \\
    Orthogonal updates (scale sweep) & 0.031--0.056 & 0.258 & 0.911--0.9999 \\
    \bottomrule
  \end{tabular}
\end{table}

\begin{figure}[t]
  \centering
  \includegraphics[width=0.85\linewidth]{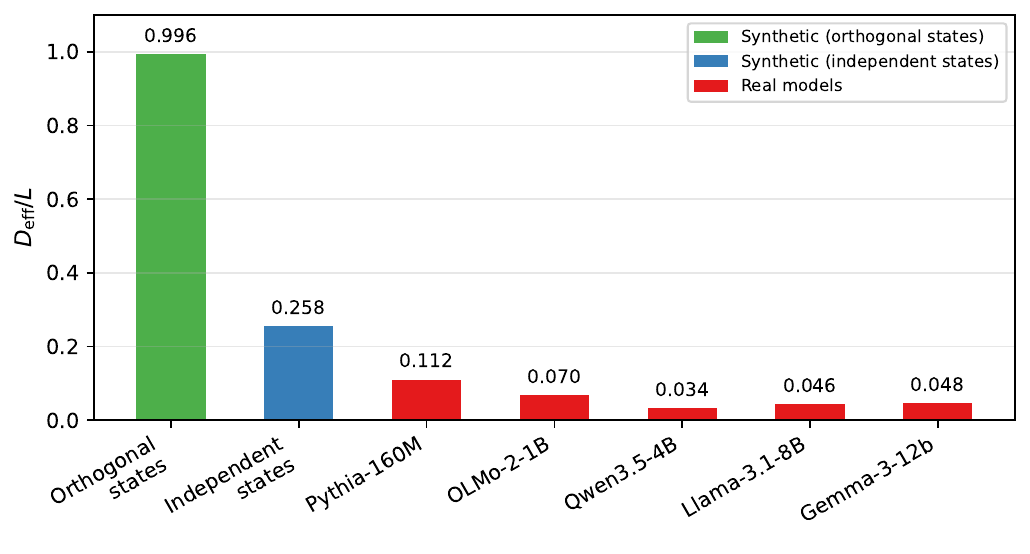}
  \caption{$\Deff/L$ for synthetic constructions (green, blue) and real
           models (red).  Orthogonal states recover $\Deff/L = 0.996$;
           independent states give 0.258; real models fall well below both,
           confirming that the low values reflect genuine redundancy in
           the accumulated residual stream.}
  \label{fig:synthetic_summary}
\end{figure}

\lead{Independent-state baseline.}
For i.i.d.\ Gaussian states $H_1, \ldots, H_L$, the empirical diagnostic gives
$\Deff/L = 0.258$ and $\hat\rho(1) = 0.093$.  This is well above the real-model
range, showing that the estimator does not collapse on unrelated states.
It is below the idealized independence limit of 1 because finite-sample linear
CKA remains positive when $d/N$ is non-negligible, even for unrelated random
matrices.

\lead{Orthogonal-state construction.}
When the states themselves are explicitly orthogonal across layers, the metric
returns $\Deff/L = 0.996$ with $\hat\rho(1) = 0.0001$.  This rules out the
failure mode in which the diagnostic is intrinsically unable to approach
full depth.

\lead{Orthogonal-update construction.}
The key control separates state redundancy from update diversity.  We construct
mutually orthogonal per-layer updates $f_\ell$ and define accumulated states by
$h_\ell = \sum_{j \leq \ell} f_j$.  In this construction the updates remain
diverse by design, yet the accumulated states can still be highly similar.
\Cref{tab:ortho_update_sweep} shows the resulting sweep over update scale.

The population version of this construction has a simple closed form.  If the
updates have common covariance and zero cross-covariance across layers, then
the update CKA is zero for distinct layers and, for $i<j$,
\[
  \CKA(h_i,h_j) =
  \frac{\|\sum_{r\leq i}\Sigma_{f_r f_r}\|_F^2}
       {\|\sum_{r\leq i}\Sigma_{f_r f_r}\|_F
        \|\sum_{r\leq j}\Sigma_{f_r f_r}\|_F}
  = \frac{i}{j}.
\]
Thus the lag-$k$ profile is
$\hat\rho(k)=\frac{1}{L-k}\sum_{i=1}^{L-k} i/(i+k)$, and direct
substitution into Eq.~(5) of the main paper gives
$\Deff=2L/(L+1)<2$.  This is why the normalized ratio $\Deff/L$ falls with
nominal depth even under maximally diverse updates
(Figure~\ref{fig:synthetic_depth_sweep}).

\begin{figure}[t]
  \centering
  \includegraphics[width=0.85\linewidth]{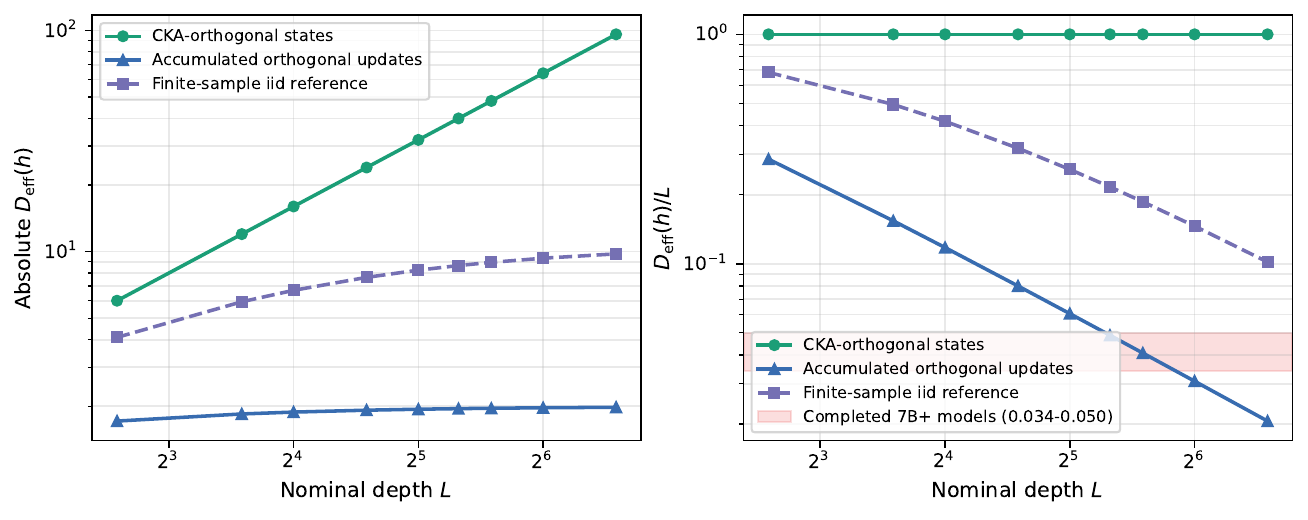}
  \caption{Depth sweep for synthetic residual constructions.
           Orthogonal states recover $\Deff=L$, whereas accumulated
           orthogonal updates obey $\Deff=2L/(L+1)<2$ and therefore
           have $\Deff/L=O(1/L)$.  The finite-sample iid reference
           matches the positive CKA bias observed in the sampled iid
           control; it is shown only as an estimator calibration.}
  \label{fig:synthetic_depth_sweep}
\end{figure}

\begin{table}[t]
  \centering
  \caption{Orthogonal-update construction ($N{=}10{,}000$, $L{=}32$,
           $d{=}1024$).  Diverse updates can coexist with very low
           accumulated-state depth.}
  \label{tab:ortho_update_sweep}
  \begin{tabular}{lccc}
    \toprule
    $\|f\|/\|h\|$ scale & $\Deff(h)/L$ & $\Deff(f)/L$ & $\hat\rho_h(1)$ \\
    \midrule
    0.01 & 0.031 & 0.258 & 0.9999 \\
    0.05 & 0.032 & 0.258 & 0.998 \\
    0.10 & 0.034 & 0.258 & 0.992 \\
    0.30 & 0.044 & 0.258 & 0.964 \\
    0.50 & 0.049 & 0.258 & 0.945 \\
    1.00 & 0.054 & 0.258 & 0.924 \\
    2.00 & 0.055 & 0.258 & 0.915 \\
    5.00 & 0.056 & 0.258 & 0.911 \\
    \bottomrule
  \end{tabular}
\end{table}

\begin{figure}[t]
  \centering
  \includegraphics[width=0.85\linewidth]{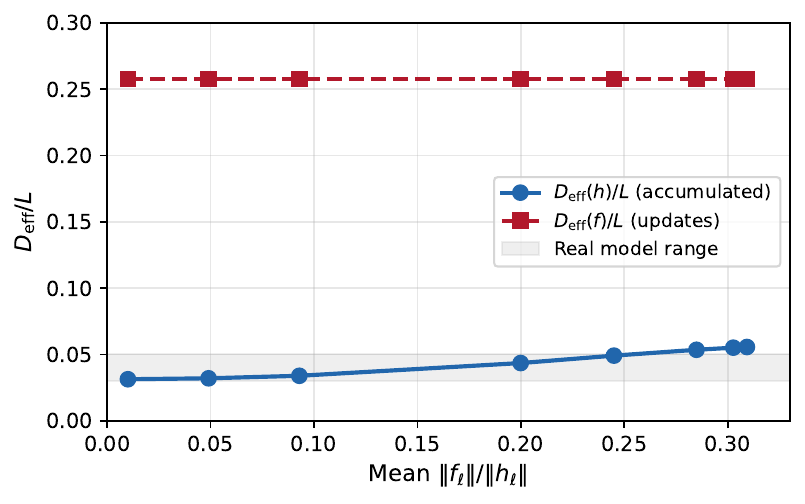}
  \caption{Orthogonal-update sweep ($N{=}10{,}000$, $L{=}32$, $d{=}1024$).
           $\Deff(f)/L$ (updates, red) remains constant at 0.258 regardless
           of scale, while $\Deff(h)/L$ (accumulated, blue) stays in the
           real-model range (gray band).  Diverse updates do not imply
           diverse accumulated states.}
  \label{fig:ortho_sweep}
\end{figure}

\lead{Implications for interpreting real models.}
These constructions establish three points.  First, $\Deff$ is not broken:
it rises substantially on independent states and approaches $L$ on
orthogonal states (Figure~\ref{fig:synthetic_summary}).  Second, low
$\Deff(h)$ does not imply that layers compute identical updates: even
perfectly orthogonal updates can yield very low accumulated-state depth
(Figure~\ref{fig:ortho_sweep}).  Third, the real-model results combine
both phenomena.  Residual accumulation keeps the sequence of states seen
by downstream layers highly similar, and the real-model updates are also
correlated, as shown by $\Deff^{\mathrm{update}}$ in
Appendix~\ref{app:update_deff}.

\section{Matched Reference Baselines}
\label{app:matched_references}

The universal $F_L=2L/(L+1)$ reference is intentionally simple: it
assumes equal-norm, mutually orthogonal updates and no persistent
initial component.  Table~\ref{tab:matched_references} asks which
model-specific quantities change that reference.  $F_{\mathrm{norm}}$
retains the measured update covariance-norm profile while setting
update cross-covariances to zero.  $F_{h_0+\mathrm{norm}}$ additionally
retains the persistent initial state.  $F_{h_0+\mathrm{corr}}$ keeps
the measured update-size profile, $h_0$, and the measured update
similarity lag profile in a scalar accumulated-state reference.

Two points matter for the interpretation.  First, matching update
sizes and $h_0$ does not remove the sub-$F_L$ finding: all sixteen
models remain below $F_{\mathrm{norm}}$ and
$F_{h_0+\mathrm{norm}}$.  Second, once the measured update-similarity
profile is retained, the sign flips for every model:
$F_{h_0+\mathrm{corr}}$ predicts lower $\Deff$ than observed.  Thus the
gap to the universal $F_L$ reference is largely explained by correlated
updates, while the observed accumulated states retain more diversity
than a scalar reference that preserves only update sizes and average
update similarities.  This is why we treat $F_L$ as a first structural
yardstick, not as a definitive null.

\begin{table}[t]
  \centering
  \caption{Matched reference baselines for $\Deff$.  $F_L$ is the
           universal equal-update reference; $F_{\mathrm{norm}}$
           matches measured update covariance norms;
           $F_{h_0+\mathrm{norm}}$ additionally retains the persistent
           initial state; $F_{h_0+\mathrm{corr}}$ also retains the
           measured update-similarity lag profile.  Gaps are relative
           to each reference, in percent.}
  \label{tab:matched_references}
  \resizebox{\linewidth}{!}{%
  \begin{tabular}{lrrrrrr}
    \toprule
    Model & $L$ & $\Deff$ & gap to $F_L$ & gap to $F_{\mathrm{norm}}$ & gap to $F_{h_0+\mathrm{norm}}$ & gap to $F_{h_0+\mathrm{corr}}$ \\
    \midrule
    Pythia-70M & 6 & 1.23 & +28.3 & +59.0 & +59.0 & -14.7 \\
    Pythia-160M & 12 & 1.34 & +27.4 & +54.5 & +54.5 & -24.6 \\
    Pythia-410M & 24 & 1.46 & +23.7 & +46.6 & +46.6 & -39.0 \\
    Pythia-1B & 16 & 1.45 & +23.0 & +34.9 & +34.9 & -38.6 \\
    Pythia-1.4B & 24 & 1.42 & +26.0 & +42.1 & +42.1 & -36.9 \\
    OLMo-2-1B & 16 & 1.12 & +40.4 & +64.1 & +61.7 & -5.8 \\
    OLMo-2-7B & 32 & 1.14 & +41.2 & +69.4 & +68.1 & -7.1 \\
    OLMo-2-13B & 40 & 1.18 & +39.7 & +63.4 & +62.3 & -9.0 \\
    Llama-3.1-8B & 32 & 1.48 & +23.9 & +40.5 & +40.5 & -40.1 \\
    Mistral-7B & 32 & 1.53 & +21.1 & +16.7 & +16.7 & -43.8 \\
    Gemma-3-12b & 48 & 2.29 & -17.1 & +48.6 & +48.6 & -89.1 \\
    Qwen3.5-0.8B & 24 & 1.07 & +44.1 & +64.4 & +63.4 & -3.0 \\
    Qwen3.5-2B & 24 & 1.08 & +43.6 & +71.6 & +71.3 & -4.2 \\
    Qwen3.5-4B & 32 & 1.08 & +44.4 & +73.6 & +73.4 & -2.9 \\
    Qwen3.5-9B & 32 & 1.15 & +40.9 & +76.8 & +76.8 & -9.0 \\
    Qwen3.5-27B & 64 & 1.35 & +31.4 & +68.7 & +68.6 & -27.7 \\
    \bottomrule
  \end{tabular}%
  }
\end{table}

\section{Full Position-0 and Norm Controls}
\label{app:control_variants_all16}

Table~\ref{tab:control_variants_all16} expands the main-text
position-0 and norm-control analysis to all sixteen models.  The
controls are symmetric: they are applied to default outliers and
non-outliers alike.  The strongest invariant result is that all
sixteen models are below the orthogonal-update reference after
token-normalisation and after top-1-PC removal.  First-non-BOS pooling
is also sub-reference for fifteen of sixteen models; the exception is
Pythia-70M, a six-layer boundary case.

\begin{table}[t]
  \centering
  \caption{Position-0 and norm control variants of $\Deff$ for the
           full model set ($N{=}10{,}000$, FineWeb-Edu).  ``Default''
           columns reproduce the canonical pooling used in
           Table~2 of the main paper.  ``First non-BOS'' replaces
           the pooled mean with position 1 only; ``token-normed''
           divides each token's hidden state by its $\ell_2$ norm
           before mean pooling; ``top-1/3 PC'' projects out dominant
           cross-sample directions from each layer's pooled matrix.
           The only above-reference model in the canonical setup,
           Gemma-3-12b, becomes strongly sub-reference under every
           control variant.}
  \label{tab:control_variants_all16}
  \resizebox{\linewidth}{!}{%
  \begin{tabular}{lcrrrrrrr}
    \toprule
    & & \multicolumn{2}{c}{Default} & \multicolumn{4}{c}{Control variants (gap to $F_L$, \%)} & \\
    \cmidrule(lr){3-4} \cmidrule(lr){5-8}
    Model & $L$ & $\Deff/L$ & gap\% & first non-BOS & token-normed & top-1 PC & top-3 PC & max pos-0 norm ratio \\
    \midrule
    Pythia-70M & 6 & 0.2047 & +28.3 & -6.5 & +30.1 & +20.0 & +11.4 & 10.0$\times$ at $L_{3}$ \\
    Pythia-160M & 12 & 0.1117 & +27.4 & +22.1 & +33.7 & +27.7 & +23.3 & 18.4$\times$ at $L_{4}$ \\
    Pythia-410M & 24 & 0.0610 & +23.7 & +8.7 & +38.4 & +34.3 & +27.9 & 42.1$\times$ at $L_{9}$ \\
    Pythia-1B & 16 & 0.0906 & +23.0 & +13.7 & +40.0 & +35.3 & +29.7 & 35.3$\times$ at $L_{6}$ \\
    Pythia-1.4B & 24 & 0.0592 & +26.0 & +14.1 & +41.8 & +37.2 & +30.9 & 24.6$\times$ at $L_{5}$ \\
    OLMo-2-1B & 16 & 0.0702 & +40.4 & +30.7 & +40.3 & +34.9 & +30.6 & 2.7$\times$ at $L_{9}$ \\
    OLMo-2-7B & 32 & 0.0357 & +41.2 & +34.5 & +41.1 & +36.4 & +33.5 & 1.5$\times$ at $L_{1}$ \\
    OLMo-2-13B & 40 & 0.0294 & +39.7 & +34.7 & +39.8 & +36.0 & +29.5 & 1.5$\times$ at $L_{1}$ \\
    Llama-3.1-8B & 32 & 0.0461 & +23.9 & +41.0 & +41.3 & +38.6 & +35.1 & 309.7$\times$ at $L_{2}$ \\
    Mistral-7B & 32 & 0.0478 & +21.1 & +33.4 & +42.0 & +34.8 & +33.0 & 202.7$\times$ at $L_{2}$ \\
    Gemma-3-12b & 48 & 0.0478 & -17.0 & +41.6 & +34.4 & +31.4 & +26.6 & 11.0$\times$ at $L_{7}$ \\
    Qwen3.5-0.8B & 24 & 0.0447 & +44.1 & +37.0 & +43.9 & +41.8 & +34.1 & 5.7$\times$ at $L_{7}$ \\
    Qwen3.5-2B & 24 & 0.0451 & +43.6 & +34.6 & +43.2 & +41.6 & +37.3 & 4.5$\times$ at $L_{15}$ \\
    Qwen3.5-4B & 32 & 0.0337 & +44.4 & +26.5 & +44.6 & +43.0 & +39.3 & 3.9$\times$ at $L_{7}$ \\
    Qwen3.5-9B & 32 & 0.0358 & +40.9 & +16.4 & +44.3 & +39.2 & +35.5 & 3.4$\times$ at $L_{7}$ \\
    Qwen3.5-27B & 64 & 0.0211 & +31.4 & +16.4 & +44.3 & +38.4 & +35.2 & 3.5$\times$ at $L_{22}$ \\
    \bottomrule
  \end{tabular}%
  }
\end{table}

Table~\ref{tab:gemma_extra_controls} extends the same controls to two
additional Gemma-3 sizes.  The default-pooling anomaly is not isolated
to Gemma-3-12b: Gemma-3-27B is also above-reference by default, while
Gemma-3-4B sits close to the reference.  In both cases, first-non-BOS,
token-normalised, and top-PC-removed variants are sub-reference.

\begin{table}[t]
  \centering
  \caption{Additional Gemma-3 control variants.  Values are gaps to
           $F_L$ in percent.  The default-pooling anomaly persists
           across larger Gemma-3 sizes, but all controls move both
           extra models into the sub-reference regime.}
  \label{tab:gemma_extra_controls}
  \begin{tabular}{lrrrrr}
    \toprule
    Model & Default & first non-BOS & token-normed & top-1 PC & top-3 PC \\
    \midrule
    Gemma-3-4B  & $+2.2$ & $+35.0$ & $+30.8$ & $+23.4$ & $+25.4$ \\
    Gemma-3-27B & $-8.3$ & $+25.5$ & $+7.4$  & $+3.0$  & $+28.0$ \\
    \bottomrule
  \end{tabular}
\end{table}

To test whether two Gemma-style block features directly account for
this behaviour, we also run Gemma-scale \emph{forward-pass}
interventions on the trained Gemma-3-4B, 12B, and 27B checkpoints.
The intervention is not retraining: post-normalisation replaces the
post-attention and post-MLP RMSNorms by identities, while no-QK
replaces the query/key RMSNorms by identities.  Table~\ref{tab:gemma_arch_ablation}
shows that post-normalisation is not the source of the above-reference
default measurement: removing it drives every Gemma size into a highly
sub-reference regime, with gaps near $+47$--$+49\%$ across controls.
QK normalisation instead modulates the position/default-pooling
channel.  Removing QK normalisation makes the default anomaly much
larger at all three sizes, and for Gemma-3-27B it even makes the
first-non-BOS and token-normalised variants above-reference.  Removing
both features is dominated by the post-normalisation ablation and again
collapses the stream.  Thus the Gemma behaviour is not explained by
post-normalisation alone; QK normalisation appears to suppress, rather
than cause, the position-0/default-pooling anomaly.

\begin{table}[t]
  \centering
  \caption{Gemma-scale forward-pass interventions.  Values are gaps to
           $F_L$ in percent under the same $N{=}10{,}000$ FineWeb-Edu
           extraction as Table~\ref{tab:gemma_extra_controls}.  ``No
           post'' replaces the post-attention and post-MLP RMSNorms by
           identities; ``no QK'' replaces the query/key RMSNorms by
           identities.}
  \label{tab:gemma_arch_ablation}
  \small
  \setlength{\tabcolsep}{3.5pt}
  \begin{tabular}{llrrrrr}
    \toprule
    Model & Intervention & Default & first non-BOS & token-normed & top-1 PC & top-3 PC \\
    \midrule
    Gemma-3-4B  & none          & $+2.2$  & $+35.0$ & $+30.8$ & $+23.4$ & $+25.4$ \\
                & no post       & $+48.1$ & $+48.5$ & $+48.3$ & $+40.8$ & $+47.5$ \\
                & no QK         & $-60.2$ & $+31.3$ & $+6.3$  & $-8.4$  & $+8.6$  \\
                & no post/no QK & $+48.2$ & $+48.5$ & $+48.4$ & $+47.7$ & $+47.6$ \\
    \midrule
    Gemma-3-12B & none          & $-17.0$ & $+35.3$ & $+34.4$ & $+31.4$ & $+26.6$ \\
                & no post       & $+47.6$ & $+48.9$ & $+48.3$ & $+47.2$ & $+46.0$ \\
                & no QK         & $-46.9$ & $+36.3$ & $+1.6$  & $+5.2$  & $+8.8$  \\
                & no post/no QK & $+46.1$ & $+48.9$ & $+46.6$ & $+45.7$ & $+40.3$ \\
    \midrule
    Gemma-3-27B & none          & $-8.3$  & $+25.5$ & $+7.4$  & $+3.0$  & $+28.0$ \\
                & no post       & $+47.5$ & $+49.1$ & $+47.1$ & $+46.8$ & $+42.9$ \\
                & no QK         & $-65.7$ & $-27.3$ & $-74.7$ & $-78.1$ & $+8.8$  \\
                & no post/no QK & $+48.1$ & $+49.1$ & $+48.1$ & $+47.3$ & $+47.3$ \\
    \bottomrule
  \end{tabular}
\end{table}

The preceding interventions act on trained checkpoints.  They therefore
show that Gemma's post-branch and QK normalisers modulate the measured
profile, but they do not by themselves show whether the anomaly is an
architectural property present before training.  We next run two
deliberately limited probes.  First, we instantiate Gemma-3-4B from its
configuration, apply the same norm interventions, and measure $\Deff$
at initialisation only.  Second, we train Gemma-3-270M variants from
scratch on a short FineWeb-Edu token stream.  These runs are not meant
to reproduce Gemma's full pretraining, instruction tuning, or
post-training pipeline.  They are controlled stress tests of whether
the norm switches are sufficient to induce the anomaly under simplified
conditions.

\begin{table}[t]
  \centering
  \caption{Gemma-3-4B initialisation-only norm probes.  Values are
           gaps to $F_L$ in percent, averaged over three random seeds
           with standard deviations.  All jobs use \texttt{max\_steps=0};
           no training is performed.  The pretrained Gemma-style
           default anomaly is absent at random initialisation.}
  \label{tab:gemma_init_probe}
  \small
  \setlength{\tabcolsep}{3.5pt}
  \resizebox{\linewidth}{!}{%
  \begin{tabular}{lrrrrrr}
    \toprule
    Variant & Default & first pos-1 & token-normed & top-1 PC & top-3 PC & max pos-0 ratio \\
    \midrule
    none          & $+43.7{\pm}0.2$ & $+43.9{\pm}0.1$ & $+43.7{\pm}0.2$ & $+42.8{\pm}0.1$ & $+42.2{\pm}0.1$ & $1.00{\pm}0.00$ \\
    no post       & $+47.1{\pm}0.1$ & $+43.6{\pm}0.1$ & $+47.1{\pm}0.1$ & $+46.0{\pm}0.1$ & $+44.7{\pm}0.2$ & $1.20{\pm}0.00$ \\
    no QK         & $+43.7{\pm}0.2$ & $+43.9{\pm}0.0$ & $+43.7{\pm}0.2$ & $+42.8{\pm}0.1$ & $+42.2{\pm}0.1$ & $1.00{\pm}0.00$ \\
    no post/no QK & $+47.1{\pm}0.1$ & $+43.6{\pm}0.1$ & $+47.1{\pm}0.1$ & $+46.0{\pm}0.1$ & $+44.7{\pm}0.2$ & $1.19{\pm}0.00$ \\
    \bottomrule
  \end{tabular}%
  }
\end{table}

Table~\ref{tab:gemma_init_probe} shows that the large-checkpoint
anomaly is not present simply because the Gemma-3-4B block contains
post-branch normalisation and QK normalisation.  All initialisation
variants are strongly sub-reference.  Removing post-normalisation
raises the gap further, while removing QK normalisation has essentially
no effect at initialisation.  This contrasts with the trained-checkpoint
intervention in Table~\ref{tab:gemma_arch_ablation}, where removing QK
normalisation amplifies the default-pooling anomaly.

\begin{table}[t]
  \centering
  \caption{Gemma-3-270M short-pretraining norm probes.  Values are final
           checkpoint gaps to $F_L$ in percent, averaged over three
           seeds with standard deviations.  These runs use a short
           FineWeb-Edu pretraining budget and are not intended to match
           the full Gemma training pipeline.  None of the variants
           reproduces the pretrained Gemma anomaly.}
  \label{tab:gemma_short_pretrain_probe}
  \small
  \setlength{\tabcolsep}{3.5pt}
  \resizebox{\linewidth}{!}{%
  \begin{tabular}{lrrrrrrr}
    \toprule
    Variant & Default & first pos-1 & token-normed & top-1 PC & top-3 PC & max pos-0 ratio & val loss \\
    \midrule
    none          & $+41.3{\pm}0.6$ & $+33.0{\pm}0.5$ & $+41.0{\pm}0.6$ & $+35.7{\pm}2.7$ & $+32.6{\pm}2.1$ & $1.48{\pm}0.04$ & $5.06{\pm}0.01$ \\
    no post       & $+40.7{\pm}0.3$ & $+22.8{\pm}1.7$ & $+40.6{\pm}0.4$ & $+36.8{\pm}1.7$ & $+33.9{\pm}3.0$ & $5.50{\pm}1.17$ & $5.08{\pm}0.01$ \\
    no QK         & $+41.7{\pm}0.3$ & $+31.3{\pm}2.3$ & $+41.6{\pm}0.3$ & $+36.2{\pm}2.9$ & $+32.9{\pm}2.7$ & $1.79{\pm}0.18$ & $5.07{\pm}0.01$ \\
    no post/no QK & $+40.7{\pm}0.3$ & $+26.7{\pm}0.9$ & $+40.8{\pm}0.3$ & $+36.6{\pm}2.0$ & $+34.3{\pm}2.8$ & $5.06{\pm}1.24$ & $5.09{\pm}0.01$ \\
    \bottomrule
  \end{tabular}%
  }
\end{table}

The short-pretraining result should be read as a boundary, not as a
negative proof about Gemma.  A few hundred million FineWeb-Edu tokens
at 270M scale are far from the full training and post-training history
of the released Gemma checkpoints.  Within this simplified setting,
however, the same norm switches are not sufficient to produce the
above-reference default behaviour: all variants remain near a
$+41\%$ sub-reference gap.  Post-normalisation removal does create a
position-0 norm spike, but the default, token-normalised, and top-PC
controlled gaps remain sub-reference.  Together with
Table~\ref{tab:gemma_init_probe}, this suggests that the pretrained
Gemma anomaly is not a pure initialisation-level architectural identity.
The safer interpretation is that Gemma's architecture modulates the
profile of trained checkpoints, while the anomaly itself likely depends
on scale, long training history, or later checkpoint-specific dynamics.

\section{Proof of CKA Convergence (Proposition~4 of the main paper)}
\label{app:cka_convergence}

\begin{proof}
Fix layers $\ell, m \in \{1, \ldots, L\}$ and let $x_1, \ldots, x_n$ be
i.i.d.\ draws from the data distribution.  Denote the $i$-th row of the
centered representation matrix at layer $\ell$ by
$\tilde h_\ell^{(i)} = h_\ell^{(i)} - \bar h_\ell$, where
$\bar h_\ell = n^{-1}\sum_{j=1}^n h_\ell^{(j)}$.

The empirical linear CKA is
\[
  \widehat{\CKA}_n(\ell, m) \;=\;
  \frac{\|\tilde H_m^\top \tilde H_\ell\|_F^2}
       {\|\tilde H_\ell^\top \tilde H_\ell\|_F \cdot
        \|\tilde H_m^\top \tilde H_m\|_F}.
\]

\textbf{Step 1: Convergence of Gram matrices.}
Each entry of $\tilde H_\ell^\top \tilde H_m / n$ is a sample covariance.
By the strong law of large numbers (the assumption
$\EE[\|h_j\|^4] < \infty$ for all layers $j$ ensures
$\mathrm{Var}(h_{\ell,j} h_{m,k}) < \infty$),
\[
  \left[\frac{\tilde H_\ell^\top \tilde H_m}{n}\right]_{jk}
  \;\xrightarrow{a.s.}\; \mathrm{Cov}(h_{\ell,j},\, h_{m,k})
  \;=:\; [\Sigma_{\ell m}]_{jk}.
\]
The sample mean $\bar h_\ell \to \mu_\ell$ a.s., so sample centering and
population centering agree asymptotically: $\tilde H_\ell^\top \tilde H_m / n
\to \Sigma_{\ell m}$ entrywise.

\textbf{Step 2: Continuous mapping.}
Define the population CKA functional:
\[
  \CKA^*(\ell, m) \;=\;
  \frac{\|\Sigma_{\ell m}\|_F^2}{\|\Sigma_{\ell\ell}\|_F \cdot \|\Sigma_{mm}\|_F}.
\]
The map $(A, B, C) \mapsto \|A\|_F^2 / (\|B\|_F \cdot \|C\|_F)$ is continuous
on the domain where $\|B\|_F, \|C\|_F > 0$.  By the non-degeneracy assumption,
$\|\Sigma_{\ell\ell}\|_F > 0$ and $\|\Sigma_{mm}\|_F > 0$.  By the continuous
mapping theorem,
$\widehat{\CKA}_n(\ell, m) \xrightarrow{p} \CKA^*(\ell, m)$.

\textbf{Step 3: Lag-averaged consistency.}
For fixed $L$ and lag $k$, the estimator $\hat\rho_{\CKA}(k)$ in
Eq.~(4) of the main paper is a finite average of $L - k$ convergent terms:
$\hat\rho_{\CKA}(k) \xrightarrow{p} \rho^*_{\CKA}(k)$, where
\[
  \rho^*_{\CKA}(k) = \frac{1}{L-k}\sum_{\ell=1}^{L-k} \CKA^*(\ell, \ell+k). \qedhere
\]
\end{proof}

\section{Maximum Lag Sensitivity}
\label{app:k_sensitivity}

Although the paper uses the full autocorrelation ($K = L{-}1$) throughout,
it is still useful to examine how $\Deff/L$ changes as a function of the
maximum included lag.  Table~\ref{tab:k_sensitivity} reports this sensitivity
for representative models.

\begin{table}[t]
  \centering
  \caption{Sensitivity of $\Deff/L$ to maximum lag $K$.  The paper uses full
           $K = L{-}1$ throughout.  For most models values stabilize by
           $K = 25$, while for Qwen3.5-27B the gap between $K = 25$ and
           $K = L{-}1$ is 0.009.}
  \label{tab:k_sensitivity}
  \resizebox{\linewidth}{!}{%
  \begin{tabular}{lcccccc}
    \toprule
    Model & $L$ & $K{=}5$ & $K{=}10$ & $K{=}15$ & $K{=}25$ & $K{=}L{-}1$ \\
    \midrule
    Pythia-70M    &  6 & 0.239 & -- & -- & -- & -- \\
    Pythia-160M   & 12 & 0.145 & 0.112 & -- & -- & 0.110 \\
    Pythia-1.4B   & 24 & 0.116 & 0.076 & 0.064 & -- & 0.057 \\
    Qwen3.5-4B    & 32 & 0.102 & 0.060 & 0.045 & 0.035 & 0.034 \\
    OLMo-2-7B     & 32 & 0.103 & 0.061 & 0.047 & 0.037 & 0.035 \\
    Llama-3.1-8B  & 32 & 0.106 & 0.064 & 0.052 & 0.046 & 0.045 \\
    OLMo-2-13B    & 40 & 0.100 & 0.058 & 0.043 & 0.032 & 0.029 \\
    Qwen3.5-27B   & 64 & 0.103 & 0.059 & 0.043 & 0.030 & 0.021 \\
    \bottomrule
  \end{tabular}%
  }
\end{table}

At $K = 5$ all models show $\Deff/L \approx 0.10$, reflecting only
short-range similarity.  As $K$ increases, $\Deff/L$ decreases monotonically
because the Bartlett taper includes progressively more high-lag terms
(which remain close to~1 for deep models).  The values stabilize by
$K = 25$ for models with $L \leq 40$; for Qwen3.5-27B ($L = 64$), the
difference between $K = 25$ and $K = 63$ is 0.009, reflecting the additional
high-lag information available in very deep models.  These results clarify
why earlier truncated calculations often appeared numerically close for
shallower models, while the full-$K$ definition remains the correct paper
setting.

\section{Alternative Aggregation Functions}
\label{app:alt_aggregations}

The weighted full-lag formula (Eq.~(5) of the main paper) is one way to aggregate the
layer similarity structure into a scalar.  To verify that the conclusions
do not depend on this specific choice, we apply three additional
aggregations to the full $L \times L$ CKA similarity matrix
$C_{ij} = \CKA(H_i, H_j)$.

\lead{Participation ratio \citep{roy2020effective}.}
Given the eigenvalues $\lambda_1 \geq \cdots \geq \lambda_L$ of $C$:
\[
  \mathrm{PR} = \frac{\bigl(\sum_i \lambda_i\bigr)^2}{\sum_i \lambda_i^2}.
\]
PR equals $L$ when all eigenvalues are equal (all layers independent) and
equals~1 when a single eigenvalue dominates (all layers identical).  It
requires no lag structure and operates on the full matrix spectrum.

\lead{Effective rank \citep{roy2007effective}.}
\[
  \mathrm{ER} = \exp\!\left(-\sum_i p_i \log p_i\right),
  \quad p_i = \frac{\lambda_i}{\sum_j \lambda_j}.
\]
ER is the exponential of the spectral entropy.  Like PR, it ranges from~1
to $L$ and captures the effective dimensionality of the similarity matrix.

\lead{Threshold counting.}
\[
  \mathrm{TC}_\tau = 1 + \sum_{\ell=1}^{L-1} \mathbf{1}[\CKA(H_\ell, H_{\ell+1}) < \tau].
\]
TC counts the number of adjacent-layer transitions where similarity drops
below a threshold $\tau$ (we use $\tau = 0.95$).  This is the simplest
possible aggregation and requires no eigendecomposition.

\lead{Results.}
Figure~\ref{fig:aggregations} compares all four aggregation methods
(Bartlett, PR, ER, TC$_{0.95}$) across thirteen models with $L \geq 16$.
Despite fundamentally different mathematical formulations, all four methods
agree on the model ordering: deeper models consistently have lower
effective depth under every aggregation.  Bartlett, PR, and ER track
closely (Spearman $\rho > 0.98$ pairwise); TC$_{0.95}$ is noisier but
preserves the same qualitative pattern.

\begin{figure}[t]
  \centering
  \includegraphics[width=0.85\linewidth]{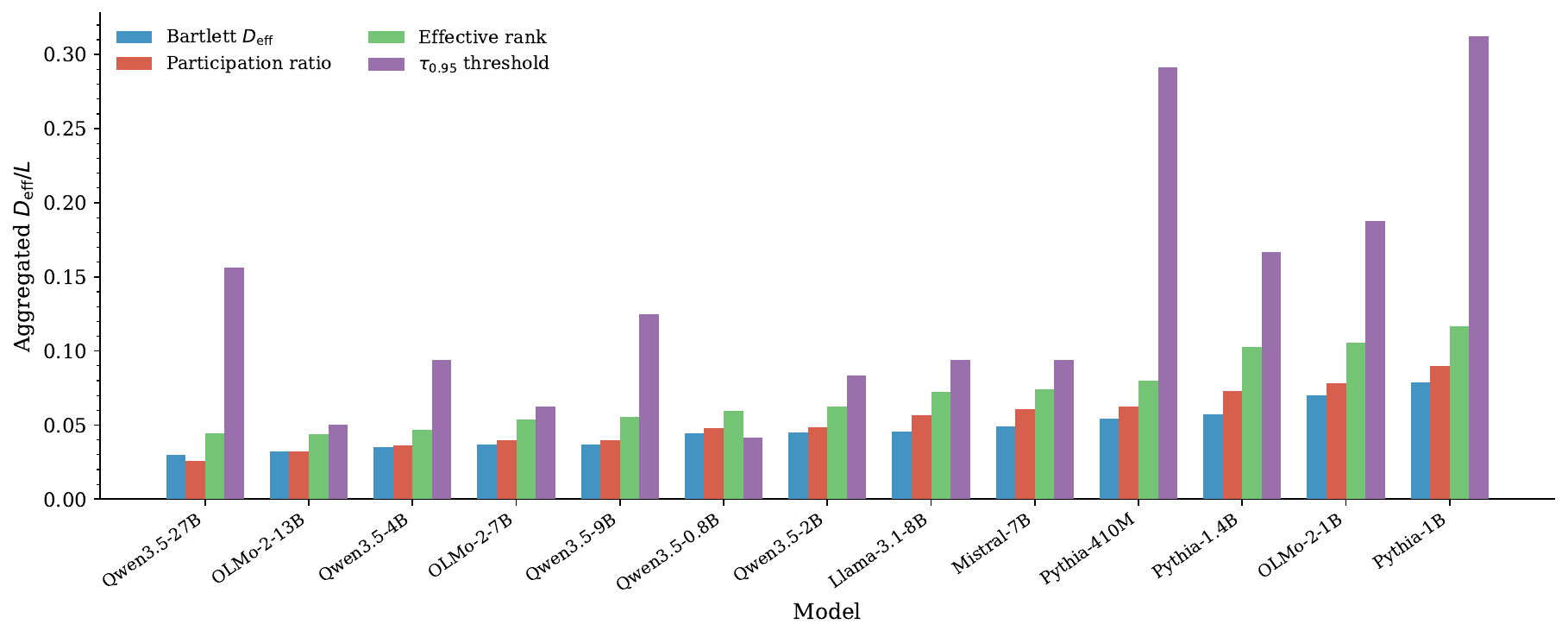}
  \caption{Four independent aggregation methods applied to the same
           CKA similarity matrices.  Despite different mathematical
           formulations, all methods agree on the model ordering and
           confirm that deep models ($L \geq 32$) have effective depth
           below 10\% of $L$.}
  \label{fig:aggregations}
\end{figure}

\section{Covariance Decay under Beta-Mixing (Norm Process)}
\label{app:mixing_connection}

The paper's main diagnostic $\Deff$ uses CKA-based layer similarity
(Definition~1 of the main paper).  For the \emph{norm-based} estimator
$\hat\rho^{\mathrm{norm}}$ (Eq.~\ref{eq:rho_norm}), we can bound the
autocovariance using the mixing coefficients of the underlying scalar process.

\begin{proposition}[Covariance bound for the norm process]
\label{prop:davydov}
Let $Z_\ell = \|h_\ell\|_2$ for $\ell = 1, \ldots, L$, and assume
$\mathrm{Var}(Z_\ell) > 0$ and $\EE[|Z_\ell|^{2+\delta}] < \infty$ for some
$\delta > 0$.  If the scalar process $(Z_\ell)$ is stationary and beta-mixing
with coefficients $\betamix(k)$, then for all $k \geq 1$:
\[
  |\mathrm{Cov}(Z_1, Z_{1+k})| \;\leq\;
  C_\delta \;\betamix(k)^{\delta/(2+\delta)} \;
  \left(\EE[|Z_1 - \EE Z_1|^{2+\delta}]\right)^{2/(2+\delta)},
\]
where $C_\delta$ is a constant depending only on $\delta$.  Dividing by
$\mathrm{Var}(Z)$ yields a bound on $|\rho^{\mathrm{norm}}(k)|$.
\end{proposition}

\begin{proof}
This is a standard covariance inequality for beta-mixing processes.  We apply
the result of~\citet[Theorem~3.3]{bradley2005basic}: for real-valued random
variables $U \in \mathcal{F}_0^t$ and $V \in \mathcal{F}_{t+k}^\infty$ with
$\EE[|U|^p] < \infty$ and $\EE[|V|^q] < \infty$ where $1/p + 1/q < 1$,
\[
  |\mathrm{Cov}(U, V)| \;\leq\; C_{p,q} \;\betamix(k)^{1 - 1/p - 1/q}\;
  \|U\|_p \;\|V\|_q.
\]
Setting $p = q = 2 + \delta$, $U = Z_1 - \EE Z_1$, $V = Z_{1+k} - \EE Z_{1+k}$
gives $1/p + 1/q = 2/(2+\delta) < 1$ (for $\delta > 0$), hence the exponent on
$\betamix(k)$ is $1 - 2/(2+\delta) = \delta/(2+\delta)$.
\end{proof}

\begin{remark}
This bound applies to the scalar norm process $(Z_\ell)$ and the norm-based
estimator $\hat\rho^{\mathrm{norm}}$.  It does not directly bound the CKA-based
$\hat\rho(k)$ used in the paper's primary $\Deff$ definition.  Establishing an
analogous bound for CKA would require a multivariate extension beyond the
present scope.
\end{remark}

\section{Alternative Tapers for $\Deff$}
\label{app:tapers}

The Bartlett taper $w(k) = 1 - k/L$ in Eq.~(5) of the main paper is one of several
standard kernel choices for spectral density estimation at zero frequency.
Alternatives include:

\lead{Parzen taper.}
\[
  w_{\mathrm{Parzen}}(k) = \begin{cases}
    1 - 6(k/M)^2 + 6(k/M)^3 & \text{if } 0 \leq k \leq M/2, \\
    2(1 - k/M)^3 & \text{if } M/2 < k \leq M, \\
    0 & \text{if } k > M,
  \end{cases}
\]
where $M \leq L$ is a bandwidth parameter.

\lead{Tukey--Hanning taper.}
\[
  w_{\mathrm{TH}}(k) = \tfrac{1}{2}\bigl(1 + \cos(\pi k / M)\bigr)
  \quad \text{for } k \leq M.
\]

All kernel choices yield $\Deff$ values that are qualitatively similar.  We
use the Bartlett taper throughout because it is a simple full-lag weighting
scheme and requires no additional bandwidth parameter.

\section{Additional Experimental Details}
\label{app:exp_details}

\lead{Hardware.}
All hidden-state extractions run on 8$\times$H100 80GB GPUs (university HPC cluster).
Analysis ($\Deff$ computation, CKA matrices, figures) runs on CPU.

\lead{Hidden-state extraction.}
Models are loaded in bfloat16.  For each input passage of up to 512 tokens,
we run a single forward pass with \texttt{output\_hidden\_states=True} and
mean-pool each layer's output over valid (non-padding) token positions.
The result is a tensor of shape $(n, L+1, d)$ per model-dataset pair, saved
as \texttt{hidden\_states.pt}.  We discard layer 0 (embedding) before computing
$\Deff$, using layers $1, \ldots, L$ as specified in Definition~1 of the main paper.

\lead{CKA computation.}
Column-centering is applied before computing linear CKA.  For numerical
stability, we divide the cross-covariance $\tilde H_m^\top \tilde H_\ell$
by $n$ before computing the Frobenius norm.  We use $N = 10{,}000$ samples
so that $N>d$ for every evaluated model; the synthetic and unbiased-CKA
checks above quantify the remaining finite-sample behaviour.

\lead{Norm-based and PCA estimators.}
The norm-based estimator uses the Frobenius norm $\|H_\ell\|_F / \sqrt{nd}$
at each layer.  The PCA estimator projects to the top $\min(64, d, n)$
principal components computed from the concatenation of all layers.

\lead{Null baseline: seed stability and figure.}
For seven architectures we evaluate random-weight $\Deff/L$ under
five independent default-Hugging-Face initialisations each
($N{=}10{,}000$, $K{=}L{-}1$).  Table~\ref{tab:null_seed_stability}
shows that random-weight $\Deff/L$ is seed-stable, with standard
deviation at most $2\times 10^{-4}$.  Trained $\Deff/L$ lies outside
the random-seed range and outside the mean $\pm 2\sigma$ band in every
case, but the direction depends on architecture: Qwen3.5 and OLMo-2-1B
move below their random-weight references, while larger OLMo-2 and
Llama-3.1-8B move above them.  Figure~\ref{fig:null_multi}
visualises Table~1 of the main paper.

\begin{table}[t]
  \centering
  \caption{Five-seed random-weight stability for $\Deff/L$
           ($N{=}10{,}000$, FineWeb-Edu, $K{=}L{-}1$).  The narrow
           random-weight bands show that the null baseline is nearly
           deterministic for a fixed architecture and initialisation
           scheme; trained values fall outside the random range in all
           seven cases.}
  \label{tab:null_seed_stability}
  \resizebox{\linewidth}{!}{%
  \begin{tabular}{lcccc}
    \toprule
    Model & Trained & Random mean & Random range & mean $\pm 2\sigma$ \\
    \midrule
    OLMo-2-1B      & 0.0702 & 0.0741 & [0.0739, 0.0743] & [0.0737, 0.0745] \\
    OLMo-2-7B      & 0.0369 & 0.0361 & [0.0360, 0.0362] & [0.0359, 0.0362] \\
    OLMo-2-13B     & 0.0294 & 0.0286 & [0.0286, 0.0287] & [0.0285, 0.0287] \\
    Qwen3.5-4B     & 0.0350 & 0.0469 & [0.0468, 0.0470] & [0.0468, 0.0471] \\
    Qwen3.5-9B     & 0.0369 & 0.0442 & [0.0441, 0.0442] & [0.0441, 0.0443] \\
    Qwen3.5-27B    & 0.0211 & 0.0217 & [0.0216, 0.0218] & [0.0216, 0.0218] \\
    Llama-3.1-8B   & 0.0468 & 0.0347 & [0.0345, 0.0348] & [0.0344, 0.0349] \\
    \bottomrule
  \end{tabular}
  }
\end{table}

\begin{figure}[t]
  \centering
  \includegraphics[width=0.85\linewidth]{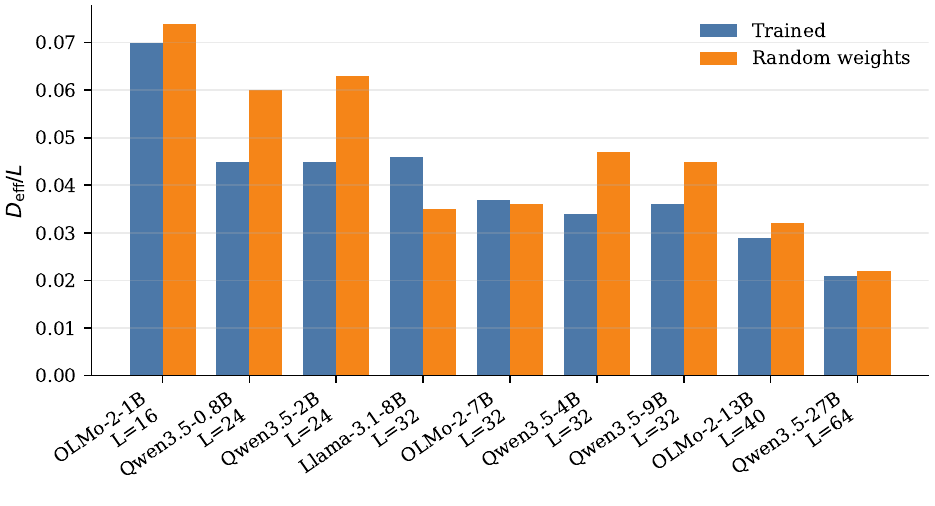}
  \caption{Trained vs.\ random-weight $\Deff/L$ across model families
           for completed pairs (sorted by depth $L$).  Bar heights
           remain in the same low range for trained and random,
           indicating a strong residual-geometric component.}
  \label{fig:null_multi}
\end{figure}

\section{Geodesic Efficiency vs.\ $\Deff/L$}
\label{app:geodesic}

We compute the geodesic efficiency
$\eta_{\mathrm{geo}} = \|h_L - h_1\|_2 / \sum_{\ell=2}^{L} \|h_\ell - h_{\ell-1}\|_2$
for the same 1{,}000 input passages and plot it against $\Deff/L$ in
\cref{fig:geo_corr}.  The scatter plot reveals a positive but moderate
correlation, indicating that the two measures capture related but distinct
aspects of layer utilisation.

\begin{figure}[t]
  \centering
  \includegraphics[width=0.6\linewidth]{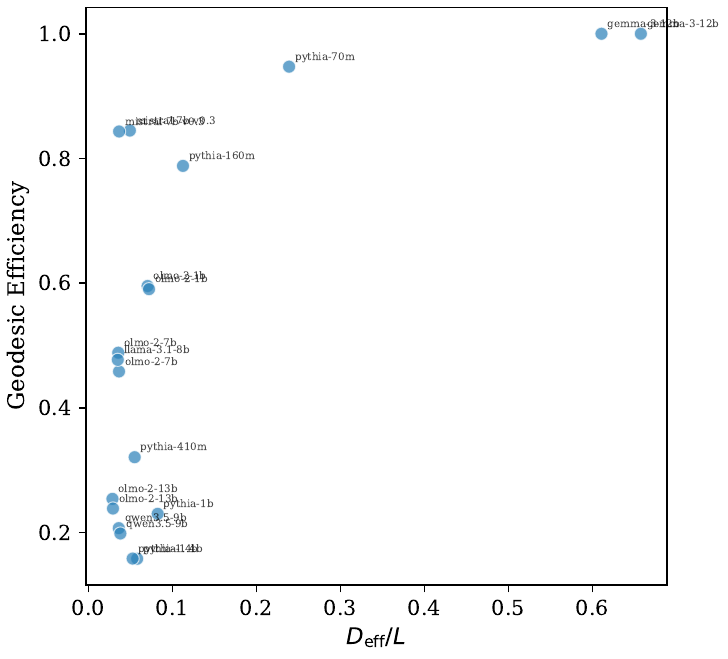}
  \caption{$\Deff/L$ vs.\ geodesic efficiency across all
           model--dataset pairs.  Pearson $r = 0.65$.  The moderate
           correlation confirms that the CKA-based diagnostic and the
           geometric measure capture related but distinct aspects of
           layer utilisation.}
  \label{fig:geo_corr}
\end{figure}

The correlation ($r = 0.65$) is positive but moderate.
Geodesic efficiency measures the \emph{directness} of the residual-stream
path, while $\Deff$ measures the \emph{statistical independence} of
adjacent layers.  A model can take a straight path (high $\eta_{\mathrm{geo}}$)
through representation space while still producing highly similar
intermediate representations (low $\Deff/L$).

\section{Norm-Based and PCA Estimators}
\label{app:estimators}

\lead{Estimator 1 (Norm-based, scalar).}
Let $s_\ell = \|H_\ell\|_F / \sqrt{nd}$ be the normalized Frobenius norm at
layer $\ell$.  Define
\begin{equation}
  \hat\rho^{\mathrm{norm}}(k)
    = \frac{\sum_{\ell=1}^{L-k}(s_\ell - \bar s)(s_{\ell+k} - \bar s)}
           {\sum_{\ell=1}^{L}(s_\ell - \bar s)^2},
  \quad \bar s = \tfrac{1}{L}\sum_{\ell=1}^L s_\ell.
  \label{eq:rho_norm}
\end{equation}
This is the cheapest estimator (one scalar per layer) but discards directional
information.

\lead{Estimator 3 (PCA projection).}
Compute a shared PCA basis $V \in \RR^{d \times r}$ from the concatenated
hidden states $[H_1; \ldots; H_L]$.  Project: $\tilde H_\ell = H_\ell V$.
Then
\begin{equation}
  \hat\rho^{\mathrm{PCA}}(k)
    = \frac{\mathrm{tr}(C(k))}{\mathrm{tr}(C(0))},
  \label{eq:rho_pca}
\end{equation}
where $C(k) = \frac{1}{L-k}\sum_{\ell=1}^{L-k} \tilde H_\ell^\top \tilde H_{\ell+k}$
is the lag-$k$ cross-covariance.

The three estimators target related but distinct population quantities:
$\hat\rho^{\mathrm{norm}}$ measures autocorrelation of a scalar summary,
$\hat\rho^{\mathrm{CKA}}$ measures average representational similarity,
and $\hat\rho^{\mathrm{PCA}}$ measures projected cross-covariance.
They agree qualitatively in practice but are not interchangeable in general.

\section{Multi-Metric Robustness}
\label{app:metric_robustness}

\begin{table}[t]
  \centering
  \caption{Auxiliary metric-robustness check: $\Deff/L$ under three
           similarity metrics on WikiText-103.
           CKA and unbiased HSIC-CKA agree within 0.002; cosine similarity
           is slightly higher but follows the same pattern.}
  \label{tab:metric_robustness}
  \begin{tabular}{lccc}
    \toprule
    Model & CKA & Unbiased HSIC-CKA & Cosine \\
    \midrule
    Pythia-70M     & 0.239 & 0.240 & 0.277 \\
    OLMo-2-1B     & 0.071 & 0.071 & 0.137 \\
    Llama-3.1-8B  & 0.035 & 0.035 & 0.051 \\
    Mistral-7B    & 0.036 & 0.036 & 0.047 \\
    OLMo-2-7B     & 0.037 & 0.038 & 0.054 \\
    OLMo-2-13B    & 0.033 & 0.033 & 0.046 \\
    Qwen3.5-9B    & 0.038 & 0.039 & 0.041 \\
    \bottomrule
  \end{tabular}
\end{table}

This auxiliary check uses WikiText-103 rather than the FineWeb-Edu calibration
set used in the main results, and is intended only to test whether the
similarity metric changes the qualitative ordering.  CKA and unbiased
HSIC-CKA are nearly identical across all models (maximum
difference 0.002), confirming that the choice between standard and unbiased
CKA is negligible in the large-sample regime relevant to the paper.  Cosine
similarity produces slightly higher
$\Deff/L$ but exhibits the same ordering and the same qualitative conclusion:
all 7B+ models have $\Deff/L < 0.055$.

\section{Instruction-Tuned vs.\ Base Models}
\label{app:instruct_base}

\begin{table}[t]
  \centering
  \caption{Instruction-tuned vs.\ base model $\Deff/L$ ($N{=}10{,}000$).
           Differences are negligible ($\leq 0.002$), indicating that
           post-training alignment leaves the residual-stream geometry largely
           unchanged.}
  \label{tab:instruct_base}
  \begin{tabular}{lcccc}
    \toprule
    Model & $L$ & Instruct & Base & $|\Delta|$ \\
    \midrule
    Qwen3.5-0.8B  & 24 & 0.045 & 0.045 & $<$0.001 \\
    Qwen3.5-4B    & 32 & 0.035 & 0.035 & $<$0.001 \\
    Llama-3.1-8B  & 32 & 0.047 & 0.047 & 0.002 \\
    \bottomrule
  \end{tabular}
\end{table}

This result is consistent with the null-baseline finding in the main text:
if random weights and trained weights already yield similar $\Deff$, it is
unsurprising that the relatively mild parameter changes introduced by
instruction tuning also leave $\Deff$ nearly unchanged.

\section{ShortGPT Analysis: Shallow Models}
\label{app:shortgpt_shallow}

ShortGPT's BI score $\BI(\ell) = 1 - \cos(h_\ell, h_{\ell+1})$ captures the
lag-1 term of the CKA autocorrelation.  Across the evaluated models, Pearson
$r > 0.97$ between mean BI and $1 - \hat\rho(1)$, confirming that BI and the
lag-1 CKA view are closely related in deep, lag-1-dominated regimes.  This is
a limitation as well as a useful boundary check: $\Deff$ is not meant to replace BI as
a local layer score.  Its added role is the calibrated model-level aggregation
of the full lag profile, together with the $F_L$ reference.

For Pythia-70M ($L = 6$), the Bartlett weights are nearly uniform across all
5 lags, and higher-lag terms contribute substantially to $\Deff$.  A
lag-1-only diagnostic would assign an incorrect effective depth to Pythia-70M
because it ignores the rapid decorrelation at higher lags ($\hat\rho(1) = 0.803$
but $\hat\rho(5) < 0.4$, as the autocorrelation profile decays appreciably
across the full 5-lag range).  The Bartlett aggregation naturally interpolates
between shallow and deep regimes: for large $L$ the taper $(1 - k/L)$
suppresses high-lag terms, recovering the lag-1-dominated regime; for small $L$
the taper is nearly flat and all lags are weighted comparably.  More broadly,
$\Deff$ correctly handles the boundary between shallow models (where most layers
contribute independently) and deep models (where lag-1 dominates).

\section{Per-Layer Update Diagnostic}
\label{app:update_deff}

The accumulated-state $\Deff$ on $h_\ell$ is the primary diagnostic in this
paper, but a complementary view computes the same Bartlett aggregation on
the per-layer updates $f_\ell = h_\ell - h_{\ell-1}$.  This update-level
diagnostic, $\Deff^{\mathrm{update}}$, asks whether the updates themselves
are correlated rather than only their accumulated sums.

\begin{table}[t]
  \centering
  \caption{Accumulated $\Deff/L$ vs.\ update-based
           $\Deff^{\mathrm{update}}/L$ ($N{=}10{,}000$, FineWeb-Edu).
           Synthetic orthogonal-update controls show that accumulation alone
           can keep $\Deff(h)$ low; this table asks whether real-model
           updates are also correlated.  They are.}
  \label{tab:update_deff}
  \resizebox{\linewidth}{!}{%
  \begin{tabular}{lccccc}
    \toprule
    Model & $L$ & $\Deff/L$ & $\Deff^{\mathrm{upd}}/L$
          & $\hat\rho^{\mathrm{upd}}(1)$ & $\|f\|/\|h\|$ \\
    \midrule
    Pythia-160M   & 12 & 0.112 & 0.154 & 0.648 & 0.54 \\
    OLMo-2-1B     & 16 & 0.070 & 0.080 & 0.846 & 0.50 \\
    Qwen3.5-4B    & 32 & 0.034 & 0.045 & 0.787 & 0.31 \\
    Llama-3.1-8B  & 32 & 0.046 & 0.045 & 0.806 & 0.37 \\
    OLMo-2-13B    & 40 & 0.029 & 0.041 & 0.796 & 0.29 \\
    Qwen3.5-27B   & 64 & 0.021 & 0.033 & 0.601 & 0.18 \\
    \bottomrule
  \end{tabular}%
  }
\end{table}

$\Deff^{\mathrm{update}}/L$ is only modestly higher than the accumulated
$\Deff/L$ and remains below $0.16$ for all completed models.  Combined
with the synthetic orthogonal-update construction
(Appendix~\ref{app:synthetic_controls}), this shows that both effects
matter: residual accumulation compresses the sequence of states seen by
downstream layers, and the updates themselves are also correlated.  This
is consistent with the ``iterative inference'' view of
\citet{jastrzebski2018residual} and the ``curse of depth'' identified by
\citet{sun2025curse}, both of which show that deeper layers produce
small, correlated updates.  The mean update-to-accumulated norm ratio
$\|f_\ell\|/\|h_\ell\|$ decreases in the deeper models, from $0.54$ at
$L{=}12$ to $0.29$--$0.37$ for $L{=}32$--$40$ and $0.18$ at $L{=}64$,
reflecting the growing dominance of the accumulated residual stream.

\section{Controlled Residual-Carry Intervention}
\label{app:residual_carry}

To probe the residual-carry mechanism causally, we train a 12-layer nanoGPT
on the \texttt{shakespeare\_char} corpus with the block equation
\begin{equation}
  h_{\ell+1} \;=\; \gamma\, h_\ell + \mathrm{attn}(\mathrm{LN}(h_\ell)),
  \qquad
  h_{\ell+1} \;=\; \gamma\, h_{\ell+1} + \mathrm{mlp}(\mathrm{LN}(h_{\ell+1})),
\end{equation}
and sweep $\gamma \in \{1.0, 0.75, 0.5, 0.25\}$ over multiple random seeds
($n_{\mathrm{layer}}{=}12$, $n_{\mathrm{embd}}{=}384$, $n_{\mathrm{head}}{=}6$,
$\mathrm{block\_size}{=}256$, $\mathrm{batch\_size}{=}64$,
$\mathrm{max\_iters}{=}5000$, $\mathrm{lr}{=}10^{-3}$ constant).  We track
$\Deff(h)$, $\Deff(\mathrm{update})$, validation loss, and the
update-to-state norm ratio at training checkpoints; full trajectories are
in Figure~\ref{fig:nanogpt_residual_carry_grid}.

\begin{table}[t]
  \centering
  \caption{Residual-carry intervention on a 12-layer nanoGPT
           (\texttt{shakespeare\_char}, iter 5000).  Values are mean across
           seeds with [min, max] in brackets ($n{=}3$ seeds for
           $\gamma \in \{1.0, 0.75\}$, $n{=}2$ for $\gamma \in \{0.5, 0.25\}$).
           At $\gamma{=}1.0$, training has already begun to overfit by iter
           5000 (best validation loss 1.56 at iter 1500).}
  \label{tab:residual_carry}
  \begin{tabular}{lcccc}
    \toprule
    $\gamma$ & val.\ loss & $\Deff(h)/L$ & $\hat\rho_h(1)$ & $\|f\|/\|h\|$ \\
    \midrule
    1.00 & 3.46 [3.30,\,3.58] & 0.094 [0.093,\,0.095] & 0.96 & 0.32 \\
    0.75 & 1.76 [1.71,\,1.86] & 0.092 [0.087,\,0.098] & 0.95 & 0.82 \\
    0.50 & 3.35 [3.35,\,3.35] & 0.124 [0.112,\,0.137] & 0.84 & 0.82 \\
    0.25 & 3.35 [3.35,\,3.35] & 0.268 [0.211,\,0.325] & 0.52 & 1.25 \\
    \bottomrule
  \end{tabular}
\end{table}

\begin{figure}[t]
  \centering
  \includegraphics[width=0.95\linewidth]{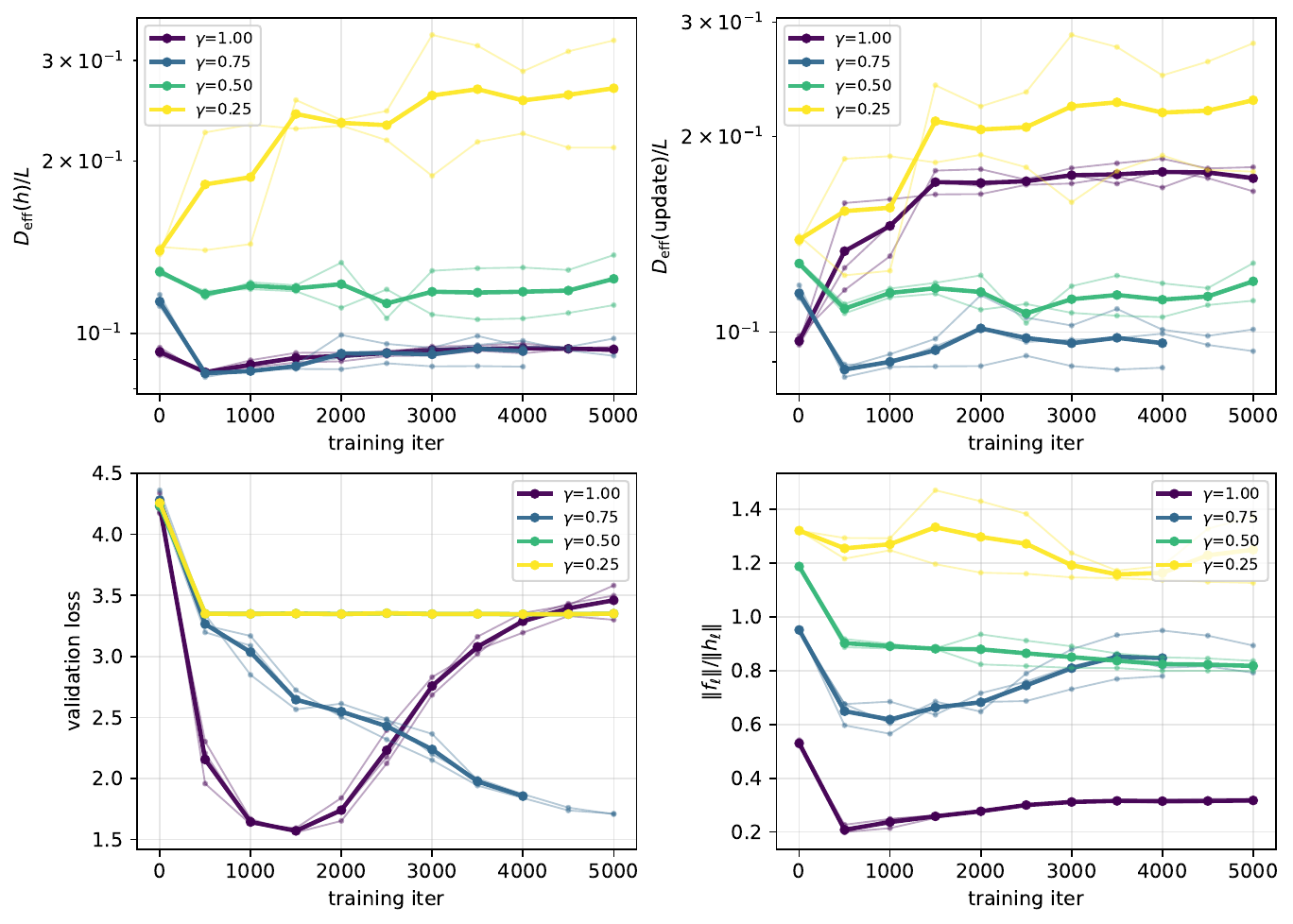}
  \caption{Controlled residual-carry intervention on a 12-layer nanoGPT.
           Bold lines are the across-seed mean at each $\gamma$; faint lines
           are individual seeds.  Top: $\Deff(h)/L$ and $\Deff(\mathrm{update})/L$
           trajectories across training, color-coded by $\gamma$.  Bottom:
           validation loss and the update-to-state norm ratio.  Substantially
           reducing $\gamma$ ($\gamma{\le}0.5$) raises $\Deff(h)/L$ from
           initialisation onward in every seed, but training fails;
           $\gamma{=}0.75$ leaves $\Deff(h)/L$ statistically indistinguishable
           from $\gamma{=}1.0$ while training slightly better.}
  \label{fig:nanogpt_residual_carry_grid}
\end{figure}

\lead{Result.} Two regimes are visible.  In the small-perturbation
regime ($\gamma{=}0.75$), $\Deff(h)/L$ is statistically indistinguishable
from $\gamma{=}1.0$ (mean $0.092$ vs $0.094$; per-seed ranges
$[0.087, 0.098]$ vs $[0.093, 0.095]$ overlap fully) and the model trains
slightly better at this iteration budget.  In the large-perturbation
regime ($\gamma \le 0.5$), $\Deff(h)/L$ rises substantially (mean $0.124$
at $\gamma{=}0.5$ and $0.268$ at $\gamma{=}0.25$, vs $0.094$ at
$\gamma{=}1.0$) in every seed, while validation loss plateaus at the
unigram baseline $\sim 3.35$ in every seed.  The same ordering between
the trains-vs-fails-to-train regimes is visible from initialisation onward
(Figure~\ref{fig:nanogpt_residual_carry_grid}).

\lead{Interpretation.} The intervention supports two compatible
claims in this controlled setting: (i) residual carry is a causal
contributor to accumulated-state effective depth, since holding
everything else fixed, large reductions in $\gamma$ shift
$\Deff(h)/L$ upward in every seed; and (ii) escaping the
low-$\Deff(h)$ regime can conflict with stable optimisation, since the
only $\gamma$ values that substantially raise $\Deff(h)/L$ also break
optimisation.  At standard $\gamma{=}1.0$, the model achieves a low
$\Deff(h)/L$ qualitatively similar to the real-LLM values from the main
experiments; mild weakening ($\gamma{=}0.75$) leaves $\Deff(h)/L$
unchanged, and stronger weakening diversifies accumulated states only at
a steep optimisation cost.  These claims are limited to this char-level
nanoGPT setup; broader claims would require additional scales and tasks
beyond this controlled experiment.

\section{Bootstrap Stability of $\Deff/L$}
\label{app:bootstrap_ci}

To verify that the $\Deff/L$ values reported in Table~2 of the main paper
are not artefacts of a particular sample of FineWeb-Edu passages, we
re-extract hidden states for five representative architectures and run
$B{=}100$ passage-level bootstrap resamples (with replacement) from the
$N{=}10{,}000$ calibration set, recomputing $\Deff/L$ on each resample.
Table~\ref{tab:bootstrap_ci} reports the bootstrap mean and the 95\%
percentile confidence interval.

\begin{table}[t]
  \centering
  \caption{Passage-bootstrap $\Deff/L$ on five representative architectures
           ($N{=}10{,}000$, $B{=}100$).  The bootstrap mean matches the
           point estimate to four decimal places, and the 95\% CI width is
           at most $3 \times 10^{-4}$ in every case.  At the three-decimal
           precision used in Table~2 of the main paper, the interval
           collapses to the reported point estimate, so we omit it there.}
  \label{tab:bootstrap_ci}
  \resizebox{\linewidth}{!}{%
  \begin{tabular}{lccccc}
    \toprule
    Model & $L$ & Point $\Deff/L$ & Bootstrap mean & 95\% CI & CI width \\
    \midrule
    Llama-3.1-8B & 32 & 0.0461 & 0.0461 & [0.0459, 0.0462] & 0.0003 \\
    OLMo-2-7B    & 32 & 0.0357 & 0.0356 & [0.0356, 0.0357] & 0.0001 \\
    Qwen3.5-4B   & 32 & 0.0337 & 0.0337 & [0.0336, 0.0337] & 0.0001 \\
    Qwen3.5-9B   & 32 & 0.0358 & 0.0358 & [0.0357, 0.0359] & 0.0002 \\
    Qwen3.5-27B  & 64 & 0.0211 & 0.0211 & [0.0210, 0.0212] & 0.0002 \\
    \bottomrule
  \end{tabular}%
  }
\end{table}

Bootstrap CI widths are tight relative to the between-model differences in
Table~2 of the main paper, so the model orderings visualised in
Figure~1 of the main paper are not affected by sample-level
Monte Carlo error.  Bootstrap resampling is
at the passage level, so these CIs reflect uncertainty over the
calibration corpus only; uncertainty over alternative model
initialisations is treated separately by the seed-robustness analysis in
Section~5 of the main paper.

\section{Pruning Boundary Checks}
\label{app:pruning_tolerance}

We do not propose $\Deff$ as a pruning method.  It is a global
accumulated-state diagnostic, whereas layer pruning is a local,
task-specific intervention.  As a boundary check, we compare a
per-layer leave-one-out diagnostic
$\Delta_\ell = \Deff(\text{full}) - \Deff(\text{remove }\ell)$
against BI by progressively skipping layers in each ranking and
measuring perplexity on a held-out FineWeb-Edu split.
Table~\ref{tab:pruning_validation} shows that BI is the correct local
tool for this question.

\begin{table}[t]
  \centering
  \caption{Boundary check: perplexity after $k$ skipped layers,
           ranked by BI vs.\ by local $\Delta\Deff$.  BI is
           substantially better at predicting safe removals on both
           models, consistent with $\Deff$ being a global
           accumulated-state diagnostic rather than a per-layer
           importance signal.  Spearman correlations between the two
           rankings are $\hat\rho_s = -0.27$ ($p = 0.31$) for
           OLMo-2-1B and $\hat\rho_s = -0.19$ ($p = 0.30$) for
           Llama-3.1-8B.}
  \label{tab:pruning_validation}
  \resizebox{\linewidth}{!}{%
  \begin{tabular}{llcccccc}
    \toprule
    Model & Method & $k{=}0$ & $k{=}1$ & $k{=}2$ & $k{=}4$ & $k{=}6$ & $k{=}8$ \\
    \midrule
    OLMo-2-1B    & BI             & 6.5 & 9.1   & 17.9    & 85.0     & 2{,}021    & 16{,}091    \\
    OLMo-2-1B    & Local $\Deff$  & 6.5 & 31.0  & 1{,}322 & 87{,}606 & 190{,}276  & 2{,}768{,}762 \\
    \midrule
    Llama-3.1-8B & BI             & 4.6 & 5.0   & 5.8     & 10.0     & 18.5       & 56.1        \\
    Llama-3.1-8B & Local $\Deff$  & 4.6 & 431.1 & 662.0   & 34{,}625 & 89{,}183   & 487{,}295   \\
    \bottomrule
  \end{tabular}%
  }
\end{table}

A separate, exploratory question is
whether \emph{model-level} $\Deff/L$ correlates with how many BI-ranked
layers a model can lose before its perplexity degrades substantially.
We test this on eight decoder-only models.  We frame this as a
diagnostic robustness check, not a predictive-validity claim.

For each model we measure perplexity on a held-out FineWeb-Edu split
under BI-ranked progressive layer skipping at $k \in \{0,1,2,4,6,8\}$
and define
\begin{equation*}
  \mathrm{tolerance\_k}_{1.10}(\mathrm{model}) =
  \max \bigl\{\, k \in \{0,1,2,4,6,8\}\,:\,
        \mathrm{PPL}_k \leq 1.10 \cdot \mathrm{PPL}_0 \,\bigr\}.
\end{equation*}
We pre-register the depth-normalized fraction
$\mathrm{tolerance\_frac}_{1.10} = \mathrm{tolerance\_k}_{1.10}/L$ as the
primary metric and the raw count as the secondary.
Table~\ref{tab:pruning_tolerance_table} reports per-model values; the
Qwen3.5-27B row is from a BI-only run that omits the local-$\Deff$
leave-one-out computation (the full-grid leave-one-out CKA loop did
not finish within wall-clock budget at $L{=}64$).

\begin{table}[t]
  \centering
  \caption{Per-model BI pruning tolerance ($N{=}10{,}000$ held-out
           FineWeb-Edu).  PPL$_0$ is the baseline; $k_{1.10}$ is the
           largest $k \in \{0,1,2,4,6,8\}$ at which
           $\mathrm{PPL}_k \leq 1.10 \cdot \mathrm{PPL}_0$;
           $f_{1.10} = k_{1.10}/L$.}
  \label{tab:pruning_tolerance_table}
  \resizebox{\linewidth}{!}{%
  \begin{tabular}{lcccccc}
    \toprule
    Model & $L$ & $\Deff/L$ & gap (\%) & PPL$_0$ & $k_{1.10}$ & $f_{1.10}$ \\
    \midrule
    OLMo-2-1B    & 16 & 0.070 & $+40.5$ & 6.54 & 0 & 0.000 \\
    Pythia-1.4B  & 24 & 0.059 & $+26.2$ & 8.11 & 0 & 0.000 \\
    Mistral-7B   & 32 & 0.048 & $+20.8$ & 3.94 & 1 & 0.031 \\
    Llama-3.1-8B & 32 & 0.046 & $+24.1$ & 4.60 & 1 & 0.031 \\
    OLMo-2-7B    & 32 & 0.036 & $+40.6$ & 5.08 & 1 & 0.031 \\
    Qwen3.5-9B   & 32 & 0.036 & $+40.6$ & 6.37 & 1 & 0.031 \\
    Qwen3.5-4B   & 32 & 0.034 & $+43.9$ & 7.17 & 2 & 0.062 \\
    Qwen3.5-27B  & 64 & 0.021 & $+31.8$ & 6.01 & 2 & 0.031 \\
    \bottomrule
  \end{tabular}%
  }
\end{table}

\begin{table}[t]
  \centering
  \caption{Correlations of BI pruning tolerance with $\Deff/L$ at the
           full $n{=}8$ panel and three leave-one-out subsets that drop
           the largest- and smallest-$\Deff/L$ endpoints.
           Pre-registered primary metric: $\mathrm{frac}_{1.10}$.
           Secondary: raw $\mathrm{tolerance\_k}_{1.10}$.  All
           $p$-values are two-sided.}
  \label{tab:pruning_tolerance_corrs}
  \resizebox{\linewidth}{!}{%
  \begin{tabular}{lc cc cc}
    \toprule
    & & \multicolumn{2}{c}{$\mathrm{frac}_{1.10}$ vs $\Deff/L$}
      & \multicolumn{2}{c}{$k_{1.10}$ vs $\Deff/L$} \\
    \cmidrule(lr){3-4} \cmidrule(lr){5-6}
    Subset & $n$ & Pearson $r$ ($p$) & Spearman $\rho$ ($p$)
                 & Pearson $r$ ($p$) & Spearman $\rho$ ($p$) \\
    \midrule
    Full panel        & 8 & $-0.74$ ($0.04$) & $-0.79$ ($0.02$)
                          & $-0.90$ ($0.002$) & $-0.93$ ($0.001$) \\
    Drop OLMo-2-1B    & 7 & $-0.59$ ($0.16$) & $-0.67$ ($0.10$)
                          & $-0.87$ ($0.01$)  & $-0.91$ ($0.005$) \\
    Drop Qwen3.5-4B   & 7 & $-0.82$ ($0.03$) & $-0.80$ ($0.03$)
                          & $-0.94$ ($0.002$) & $-0.91$ ($0.005$) \\
    Drop both         & 6 & $-0.68$ ($0.14$) & $-0.66$ ($0.15$)
                          & $-0.92$ ($0.009$) & $-0.86$ ($0.03$) \\
    \bottomrule
  \end{tabular}%
  }
\end{table}

\lead{Reading.}
At the full $n{=}8$ panel both tolerance metrics correlate negatively
with $\Deff/L$: models with lower $\Deff/L$ tolerate more BI-pruned
layers.  The pre-registered primary metric, the depth-normalized
fraction $\mathrm{frac}_{1.10}$, passes the strong-correlation
threshold ($|r|, |\rho| \geq 0.7$, $p \leq 0.05$) at $n{=}8$ but does
\emph{not} survive the leave-one-out check that drops OLMo-2-1B
($|r|=0.59$, $p=0.16$).  The single-point sensitivity reflects a
measurement-granularity issue: the BI-pruning grid samples only
$k \in \{0,1,2,4,6,8\}$, so the integer $k$ values are tightly
clustered while $L$ varies $16$--$64$, and dividing by $L$ amplifies
the $L$-dependent noise.  The raw secondary metric
$\mathrm{tolerance\_k}_{1.10}$ is robust to all leave-one-out subsets
($|r| \geq 0.87$, $|\rho| \geq 0.86$, $p \leq 0.03$ in every case),
but because the fraction was the pre-registered primary, we do not
promote this to a main-text claim.

We treat the model-level pruning-tolerance check as suggestive rather
than conclusive: raw skipped-layer tolerance correlates strongly with
$\Deff/L$, but the depth-normalized tolerance fraction is less robust.
A finer $k$-grid (e.g.\ $k \in \{0,1,2,3,4,5,6,8\}$) would reduce the
granularity-induced noise in the fraction metric and is left as future
work.  The paper's main pruning claim
(Section~5 of the main paper,
Table~\ref{tab:pruning_validation}) remains that local $\Delta\Deff$
is a worse per-layer pruning ranker than BI, consistent with $\Deff$
being a global accumulated-state diagnostic rather than a per-layer
importance score.

\lead{Other predictors.}
For completeness, neither the absolute $\Deff$ nor the gap to $F_L$ are
useful predictors of BI pruning tolerance at any subset (all
$p > 0.5$).  This is consistent with the role we assign these
quantities throughout the paper: $\Deff/L$ is the relevant
\emph{redundancy} summary, while gap-to-$F_L$ is a regime indicator
that does not directly encode tolerance.

\section{Capability-Relevant Scaling Check}
\label{app:capability_correlation}

We include a deliberately conservative benchmark comparison because
one natural question is whether the global redundancy diagnostic is
also a capability predictor.  The answer from this small panel is:
\emph{not as a main claim}.  $\Deff$ is not intended to be a capability
score, and benchmark comparisons are confounded by model scale,
training data, architecture, tokenizer, and evaluation details.  We
therefore report the results only to delimit the diagnostic's scope:
all sixteen models, a modern-only subset excluding Pythia, and the
within-family Qwen3.5 panel.  OLMo-2 has only three evaluated sizes, so
within-family correlations are underpowered and omitted.  Even the
within-Qwen3.5 correlations should be read as co-movement along a
shared scaling axis, not as evidence that $\Deff/L$ predicts capability
beyond parameter count or depth.

\begin{table}[t]
  \centering
  \caption{Capability-evaluation coverage.  MMLU, ARC-Challenge, and
           HellaSwag are available for all 16 models.  GSM8K-CoT is
           evaluated only for models with at least 4B parameters; the
           Qwen3.5-27B and Gemma-3-12b GSM8K values exhibit suspicious
           answer-extraction failures and are not used for claims.}
  \label{tab:capability_coverage}
  \resizebox{\linewidth}{!}{%
  \begin{tabular}{llcccc}
    \toprule
    Family & Models & MMLU & ARC & HellaSwag & GSM8K-CoT \\
    \midrule
    Pythia & 70M, 160M, 410M, 1B, 1.4B & 5/5 & 5/5 & 5/5 & 0/5 \\
    OLMo-2 & 1B, 7B, 13B & 3/3 & 3/3 & 3/3 & 2/3 \\
    Qwen3.5 & 0.8B, 2B, 4B, 9B, 27B & 5/5 & 5/5 & 5/5 & 3/5 \\
    Llama & 3.1-8B & 1/1 & 1/1 & 1/1 & 1/1 \\
    Mistral & 7B & 1/1 & 1/1 & 1/1 & 1/1 \\
    Gemma & 3-12b & 1/1 & 1/1 & 1/1 & 1/1 \\
    \midrule
    Total & 16 models & 16/16 & 16/16 & 16/16 & 8/8 eligible \\
    \bottomrule
  \end{tabular}%
  }
\end{table}

\begin{table}[t]
  \centering
  \caption{Exploratory correlation between benchmark score and $\Deff/L$.
           Values report Pearson $r$ and Spearman $\rho$ with two-sided
           $p$-values.  All-model correlations are informative but
           confounded by family and era; the within-Qwen3.5 panel is the
           cleanest within-family scaling check, but remains collinear
           with scale and depth and is not a predictive-validity claim.}
  \label{tab:capability_corrs}
  \resizebox{\linewidth}{!}{%
  \begin{tabular}{llcccc}
    \toprule
    Panel & Task & Pearson $r$ & $p$ & Spearman $\rho$ & $p$ \\
    \midrule
    All 16 & MMLU      & $-0.640$ & $0.008$ & $-0.931$ & $<0.001$ \\
    All 16 & ARC       & $-0.664$ & $0.005$ & $-0.842$ & $<0.001$ \\
    All 16 & HellaSwag & $-0.718$ & $0.002$ & $-0.794$ & $<0.001$ \\
    All 16 & GSM8K-CoT & $-0.291$ & $0.484$ & $-0.349$ & $0.396$ \\
    \midrule
    Modern only & MMLU      & $-0.749$ & $0.008$ & -- & -- \\
    Modern only & ARC       & $-0.600$ & $0.051$ & -- & -- \\
    Modern only & HellaSwag & $-0.426$ & $0.192$ & -- & -- \\
    Modern only & GSM8K-CoT & $-0.291$ & $0.484$ & -- & -- \\
    \midrule
    Qwen3.5 & MMLU      & $-0.896$ & $0.040$ & $-0.872$ & $0.054$ \\
    Qwen3.5 & ARC       & $-0.913$ & $0.030$ & $-0.872$ & $0.054$ \\
    Qwen3.5 & HellaSwag & $-0.864$ & $0.059$ & $-0.872$ & $0.054$ \\
    Qwen3.5 & GSM8K-CoT & \multicolumn{4}{c}{$n=3$, skipped} \\
    OLMo-2 & all tasks & \multicolumn{4}{c}{$n=3$, skipped} \\
    \bottomrule
  \end{tabular}%
  }
\end{table}

\begin{table}[t]
  \centering
  \caption{Leave-one-family-out robustness for all-16 correlations
           between score and $\Deff/L$.  Entries are Pearson $r$ with
           two-sided $p$-values in parentheses.  MMLU and ARC remain
           negative under every family deletion; HellaSwag weakens after
           dropping Pythia, indicating an era/family confound.}
  \label{tab:capability_loo}
  \resizebox{\linewidth}{!}{%
  \begin{tabular}{lcccc}
    \toprule
    Dropped family & MMLU & ARC & HellaSwag & GSM8K-CoT \\
    \midrule
    Gemma   & $-0.696$ ($0.004$) & $-0.730$ ($0.002$) & $-0.785$ ($0.001$) & $-0.021$ ($0.964$) \\
    Llama   & $-0.636$ ($0.011$) & $-0.663$ ($0.007$) & $-0.719$ ($0.003$) & $-0.305$ ($0.505$) \\
    Mistral & $-0.637$ ($0.011$) & $-0.663$ ($0.007$) & $-0.724$ ($0.002$) & $-0.190$ ($0.683$) \\
    OLMo-2  & $-0.617$ ($0.025$) & $-0.638$ ($0.019$) & $-0.708$ ($0.007$) & $-0.197$ ($0.709$) \\
    Pythia  & $-0.749$ ($0.008$) & $-0.600$ ($0.051$) & $-0.426$ ($0.192$) & -- \\
    Qwen3.5 & $-0.564$ ($0.071$) & $-0.606$ ($0.048$) & $-0.699$ ($0.017$) & $-0.848$ ($0.070$) \\
    \bottomrule
  \end{tabular}%
  }
\end{table}

\lead{Reading.}
The capability check is mainly a negative scope result.  Within the
Qwen3.5 family, lower $\Deff/L$ co-moves with higher MMLU and
ARC-Challenge performance ($r=-0.896$ and $r=-0.913$), with HellaSwag
just below the pre-specified significance threshold.  But the same
within-family axis also changes parameter count and depth, so the
correlation does not establish predictive information beyond scale.
Cross-family modern-only correlations are mixed, HellaSwag is partly
driven by the Pythia-vs-modern era gap, and GSM8K-CoT is unreliable in
this run because several generated-answer extraction scores are
anomalous.  We therefore do not use capability prediction as evidence
for $\Deff$.  The main role of $\Deff$ remains the one used throughout
the paper: diagnosing global residual-stream redundancy and comparing
residual-stream operating regimes.


\section{Alignment-free similarity check: a Gromov-Wasserstein pilot}
\label{app:gw_pilot}

\lead{Motivation.}
$\CKA$ measures second-order geometric similarity on the same samples, and its
invariances are a known limitation. As an alignment-free cross-check, we
compare the layer-pair structure induced by the Gromov-Wasserstein (GW)
distance, which first aligns the metric structure of two point clouds by
optimal transport and therefore does not assume a shared coordinate system,
against the $\CKA$ structure used in the paper.

\lead{Setup.}
For each model we subsample 256 evaluation passages (identical indices across
layers). For each layer we form the intra-cloud Euclidean distance matrix of
the 256 pooled states, normalized by its mean so only relative metric
structure matters. For every layer pair we compute the squared-loss GW
distance with the exact conditional-gradient solver of the Python Optimal
Transport (POT) library. We then correlate the resulting layer-pair GW matrix
with the $\CKA$ matrix over off-diagonal pairs (embedding row dropped).

\begin{table}[h]
\centering
\caption{Gromov-Wasserstein pilot. GW distance is strongly anti-correlated
with $\CKA$ over layer pairs, and increases with layer distance, so the
regime described in the paper is visible under an alignment-free optimal
transport lens as well.}
\label{tab:gw_pilot}
\begin{tabular}{lccc}
\toprule
Model & $\rho_s(\mathrm{GW}, \CKA)$ & $\rho_s(|i{-}j|, \mathrm{GW})$ & $\rho_s(|i{-}j|, \CKA)$ \\
\midrule
Pythia-410M  & $-0.86$ & $+0.45$ & $-0.48$ \\
Llama-3.1-8B & $-0.68$ & $+0.54$ & $-0.89$ \\
OLMo-2-7B    & $-0.47$ & $+0.44$ & $-0.82$ \\
\bottomrule
\end{tabular}
\end{table}

\lead{Reading.}
The qualitative picture is concordant: where $\CKA$ reports high cross-layer
similarity, GW reports small metric distortion, and GW grows with layer
distance. The quantitative differences confirm that the two views capture
different aspects of representational geometry. We nevertheless retain linear
$\CKA$ for the diagnostic itself because the orthogonal-update reference
$F_L = 2L/(L+1)$ exists in closed form for it; to our knowledge no analogous
closed-form reference under residual accumulation is available for GW, so a
GW-based effective depth would currently be uncalibrated. Extending the
calibration theory to transport-based similarities is an interesting
direction for future work.

\end{document}